\documentclass[preprint]{imsart}

\usepackage{amsbsy}
\usepackage{amsmath}
\usepackage{amsthm}

\usepackage{color}
\usepackage[all]{xy}

\usepackage{wrapfig}

\usepackage{epsf}
\usepackage{epsfig}

\usepackage{enumitem}

\usepackage{bbm}

\newtheorem{theorem}{Theorem}

\newtheorem{lemma}{Lemma}

\usepackage[vlined,linesnumberedhidden,ruled,resetcount]{algorithm2e}

\SetKwBlock{Repeat}{repeat}{}
\SetKwInOut{Initialization}{Initialization}

\newcommand{\bfmX}{\boldsymbol{X}}
\newcommand{\bfmW}{\boldsymbol{W}}
\newcommand{\bfmZ}{\boldsymbol{Z}}
\newcommand{\bfmA}{\boldsymbol{A}}
\newcommand{\bfmB}{\boldsymbol{B}}
\newcommand{\ci}{\perp\!\!\!\perp}
\newcommand{\nci}{\not\perp\!\!\!\perp}

\newcommand{\bfGamma}{\boldsymbol{\Gamma}}
\newcommand{\bfI}{{\bf I}}
\newcommand{\bfOmega}{\boldsymbol{\Omega}}

\newcommand{\bfSigma}{\boldsymbol{\Sigma}}
\newcommand{\bfBeta}{\boldsymbol{\beta}}

\newcommand{\bfA}{{\bf A}}
\newcommand{\bfU}{{\bf U}}
\newcommand{\bfD}{{\bf D}}
\newcommand{\bfV}{{\bf V}}
\newcommand{\bfX}{{\bf X}}
\newcommand{\bfY}{{\bf Y}}
\newcommand{\bfW}{{\bf W}}

\newcommand{\bfmU}{\boldsymbol{U}}

\newcommand{\beginsupplement}{%
        \setcounter{table}{0}
        \renewcommand{\thetable}{S\arabic{table}}%
        \setcounter{figure}{0}
        \renewcommand{\thefigure}{S\arabic{figure}}%
     }

\begin{document}

\begin{frontmatter}

\title{Causality-aware counterfactual confounding adjustment as an alternative to linear residualization in anticausal prediction tasks based on linear learners}

\runtitle{Causality-aware confounding adjustment versus linear residualization}

\author{Elias Chaibub Neto \\ Sage Bionetworks, Seattle, WA 98121}

\runauthor{Chaibub Neto E.}

\begin{abstract}
Linear residualization is a common practice for confounding adjustment in machine learning (ML) applications. Recently, causality-aware predictive modeling has been proposed as an alternative causality-inspired approach for adjusting for confounders. The basic idea is to simulate counterfactual data that is free from the spurious associations generated by the observed confounders. In this paper, we compare the linear residualization approach against the causality-aware confounding adjustment in anticausal prediction tasks, and show that the causality-aware approach tends to (asymptotically) outperform the residualization adjustment in terms of predictive performance in linear learners. Importantly, our results still holds even when the true model is not linear. We illustrate our results in both regression and classification tasks, where we compared the causality-aware and residualization approaches using mean squared errors and classification accuracy in synthetic data experiments where the linear regression model is mispecified, as well as, when the linear model is correctly specified. Furthermore, we illustrate how the causality-aware approach is more stable than residualization with respect to dataset shifts in the joint distribution of the confounders and outcome variables.
\end{abstract}

\end{frontmatter}

\section{Introduction}

Confounding is a ubiquitous problem in machine learning (ML). While the precise definition of confounding varies across the many applied fields that are plagued by this critical issue~\cite{morabia2011}\cite{pearl2018}, here, we subscribe to the causality-based definition of confounding. Following~\cite{pearl2009}, we adopt a graphical definition where a variable $A$ is a confounder of the relationship between variables $X$ and $Y$, if there is an active path from $A$ to $X$ that does not go through $Y$, and an active path from $A$ to $Y$ that does not go through $X$, in the causal graph describing how these variables are related.

Linear residualization is a common technique for confounding adjustment in applied ML work. The basic idea is to regress the input data on the observed confounders and use the residuals of the regression fits as the new inputs for ML algorithms. As pointed in~\cite{snoek2019}, in neuroimage studies, linear residualization (also described as ``confounding regression", ``image correction", or simply as ``regressing out" confounding effects) is a widely used and perhaps the most common approach employed in practice for addressing confounding~\cite{abdulkadir2014,dubois2017,dukart2011,kostro2014,rao2017,tood2013,greenstein2012,doan2017,friston1994,maglanoc2020}.

When training predictive models from neuroimage data, researchers typically have access to other demographic variables such as age and gender. Such variables often affect the imaging data, as well as, the outputs the researchers are trying to predict, and represent confounders of the prediction. In this paper, we compare the linear residualization approach against the recently proposed causality-aware confounding adjustment~\cite{achaibubneto2020a} in the particular context of in anticausal prediction tasks (where the output has a causal influence on the inputs). (Note that some of the applications described above represent anticausal prediction tasks. For instance, in neuroimage diagnostic applications the images capture symptoms of the disease, as described in more detail in Supplementary Section 7). The basic idea behind the causality-aware approach is to simulate counterfactual data that is free from the spurious associations generated by the observed confounders. In anticausal predictions tasks, the approach is implemented by regressing each input on both the confounders and output, and then generating counterfactual inputs by adding back the estimated residuals to a linear predictor that no longer includes the confounder variables. The new counterfactual inputs are then used as the inputs for the ML algorithm.

In this paper, we prove that, for anticausal prediction tasks, the strength of the covariance between the causality-aware counterfactual features and the output variable is always asymptotically stronger than the covariance between the residualized features and the output. Since this result holds for all features/inputs, we conjecture that the causality-aware approach asymptotically outperforms the linear residualization adjustment in terms of predictive performance in linear ML models (and we prove this result for the mean squared error metric in some particular cases). Importantly, our covariance strength result still holds even when the true data generating process is not linear, so that the linear models used to adjust for confounding are mispecified.

We illustrate our analytical results using both regression and classification tasks using both correct and mispecified models. For the regression task, we compared the causality-aware and residualization approaches using mean squared errors, while the classification task performance is compared using classification accuracy.

Finally, while our analytical results assume the absence of dataset shifts between the training and test sets, we illustrate how the causality-aware approach is more stable than residualization under dataset shifts of the joint distribution of the confounders and outcome variables. Our results show that linear residualization can be safely replaced by the causality-aware approach in ML applications based on linear learners.

\section{Background}

\subsection{Notation and causality definitions}

Throughout the text, we let $X$, $Y$, and $A$ represent, respectively, the input, output and confounder variables. Sets of random variables are represented in italic and boldface, and we use the superscripts $tr$ and $ts$ to represent the training and test sets, respectively. We adopt Pearl's mechanism-based approach to causation~\cite{pearl2009} where the joint distribution of a set of variables is accompanied by a \textit{directed acyclic graph} (DAG), also denoted as a \textit{causal diagram/graph}, representing our prior knowledge (or assumptions) about the causal relation between the variables. The \textit{nodes} on the causal graph represent the random variables, and the \textit{directed edges} represent causal influences of one variable on another. Throughout the text we assume that the variables $X$, $Y$, and $A$ have been standardized to have mean 0 and variance 1\footnote{Note that any linear model $V^o_s = \mu_s + \Sigma_{j\not=s} \beta_{sj} V^o_j + W^o_s$, where $V^o_s$ represents the original data, can be reparameterized into its equivalent standardized form $V_s = \sum_{j\not=s} \gamma_{sj} V_j + W_s$, where $V_s = (V^o_s - E(V^o_s))/Var(V^o_s)^{\frac{1}{2}}$ represent standardized variables with $E(V_s) = 0$ and $Var(V_s) = 1$; $\gamma_{_{{V_s}{V_j}}} = \beta_{{V_s}{V_j}} (Var(V^o_j)/Var(V^o_s))^{\frac{1}{2}}$ represent the path coefficients~\cite{wright1934}; and $W_s = W^o_s/Var(V^o_s)^{\frac{1}{2}}$ represent the standardized error terms.} (with the exception of the stability analyses).

\subsection{The confounded anticausal prediction task}

\begin{wrapfigure}{r}{0.37\textwidth}
\vskip -0.1in
{\scriptsize
$$
\xymatrix@-2.2pc{
& W_{A_1} \ar[drr] \ar@/^1pc/@{<->}[rr] & \ldots & W_{A_k} \ar[drr] &&& \\
W_{X_1} \ar[ddr] \ar@/_1pc/@{<->}[dd] & & & *+[F-:<10pt>]{A_1} \ar[ddll] \ar[ddddll] \ar[dddrrr] & \ldots & *+[F-:<10pt>]{A_k} \ar[ddllll] \ar[ddddllll] \ar[dddr] \\
\vdots &&&&&&  \\
W_{X_p} \ar[ddr] & *+[F-:<10pt>]{X_1} &&&& \\
& \vdots & & & & & *+[F-:<10pt>]{Y} \ar[ulllll] \ar[dlllll] \\
& *+[F-:<10pt>]{X_p} &&& W_Y \ar[urr] && \\
}
$$}
\vskip -0.1in
  \caption{Confounded anticausal prediction task.}
  \label{fig:confounded.anticausal.example}
  \vskip -0.1in
\end{wrapfigure}
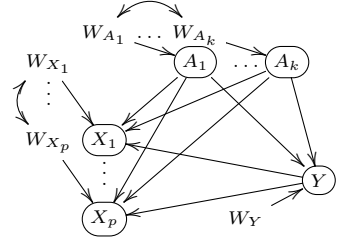
A prediction task where the output variable has a causal influence on the input variables is denoted an \textit{anticausal prediction task}~\cite{scholkopf2012}. Furthermore, if there are variables $\bfmA$ such that for each element $A_j$ of $\bfmA$ there is an active path from $A_j$ to $Y$ that does not go through any element of $\bfmX$ and, similarly, there are active paths from $A_j$ to elements of $\bfmX$ that do not go through $Y$, then we have a \textit{confounded anticausal prediction task}, as illustrated in Figure \ref{fig:confounded.anticausal.example}. Note that the variables $\{W_{X_1}, \ldots, W_{X_p}\}$ and $\{W_{A_1}, \ldots, W_{A_k}\}$ represent sets of correlated error terms, and that the causal model in Figure \ref{fig:confounded.anticausal.example} might represent a reparameterization of a model with uncorrelated error terms and unknown causal relations among the $\bfmX$ inputs, as well as, among the $\bfmA$ confounders.

This point has been described in detail in~\cite{achaibubneto2020a}, where it has been shown that, in the special case where the true data generation process corresponds to linear structural causal models, we can always reparameterize the original model in a way where the covariance structure among the input variables, as well as, the covariance structure among the confounder variables is pushed towards the respective error terms as illustrated in Figure \ref{fig:confounded.anticausal.example}. However, it is important to clarify, that even when the true data generation process does not correspond to a set of linear structural equations, we can still model the data according to the diagram in Figure \ref{fig:confounded.anticausal.example}, with the understanding that we are working with a mispecified model. In either way, we model the input variables, $X_j$, $j = 1, \ldots, p$, according to the linear structural equations,
\begin{equation}
X_j = \sum_{i=1}^k \gamma_{_{{X_j}{A_i}}} A_i + \gamma_{_{{X_j}{Y}}} Y +  W_{X_j}
\end{equation}
which can be represented in matrix form by,
\begin{equation}
\underbrace{
\begin{pmatrix}
X_1 \\
\vdots \\
X_p \\
\end{pmatrix}}_{\bfmX}
=
\underbrace{
\begin{pmatrix}
\gamma_{_{{X_1}{A_1}}} & \ldots & \gamma_{_{{X_1}{A_k}}} \\
\vdots & \ddots & \vdots \\
\gamma_{_{{X_p}{A_1}}} & \ldots & \gamma_{_{{X_p}{A_k}}} \\
\end{pmatrix}}_{\bfGamma_{XA}}
\underbrace{
\begin{pmatrix}
A_1 \\
\vdots \\
A_k \\
\end{pmatrix}}_{\bfmA}
+
\underbrace{
\begin{pmatrix}
\gamma_{_{{X_1}{Y}}} \\
\vdots \\
\gamma_{_{{X_p}{Y}}} \\
\end{pmatrix}}_{\bfGamma_{XY}} \, Y
+
\underbrace{
\begin{pmatrix}
W_{X_1} \\
\vdots \\
W_{X_p} \\
\end{pmatrix}}_{\bfmW_X}~.
\end{equation}

Similarly, we model the output variable, $Y$, as,
\begin{equation}
Y = \sum_{j=1}^k \gamma_{_{{Y}{A_j}}} A_j +  W_Y =
\underbrace{
\begin{pmatrix}
\gamma_{_{{Y}{A_1}}} & \ldots & \gamma_{_{{Y}{A_k}}} \\
\end{pmatrix}}_{\bfGamma_{YA}}
\underbrace{
\begin{pmatrix}
A_1 \\
\vdots \\
A_k \\
\end{pmatrix}}_{\bfmA}
+ \, W_Y~,
\end{equation}
so that our inferences will be based on the potentially mispecified models,
\begin{align}
\bfmX &= \bfGamma_{XA} \, \bfmA  + \bfGamma_{XY} \, Y + \bfmW_X~, \label{eq:X.model} \\
Y &= \bfGamma_{YA} \, \bfmA + W_Y~, \label{eq:Y.model}
\end{align}
where the variables $\bfmX$, $\bfmA$, and $Y$ are scaled to have mean 0 and variance 1, and the error terms have mean 0 and finite variance (but are not assumed to be Gaussian.)

\subsection{A note on non-representative development data and target population dataset shifts}

In this paper, we assume that a modeler/researcher has access to independent and identically distributed (i.i.d.) training and test sets derived from a population that is not representative of the target populations where the ML algorithm will be deployed. This scenario represents the all too common situation of a researcher interested in developing ML models, but with access to a single (and likely non-representative) development dataset, which the researcher splits into i.i.d. training and test sets in order to train and evaluate the learners.

While in this setting researchers sometimes observe that confounding adjustment decreases the predictive performance (e.g., when the confounder increases the association between the inputs and the output variable), researchers are still usually willing to perform confounding adjustments in order to obtain more stable predictions that will not degrade (or, at least, degrade to a lesser degree) when applied to distinct target populations that are shifted relative to the non-representative development population.

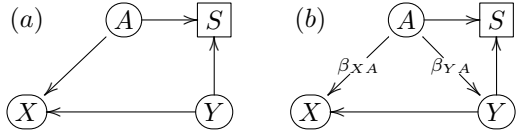
\begin{wrapfigure}{r}{0.6\textwidth}
\vskip -0.1in
$$
\xymatrix@-0.3pc{
(a) & *+[F-:<10pt>]{A} \ar[dl] \ar[r] & *+[F]{S} & (b) & *+[F-:<10pt>]{A} \ar[dl]|-{\beta_{XA}} \ar[dr]|-{\beta_{YA}} \ar[r] & *+[F]{S} \\
*+[F-:<10pt>]{X}  & & *+[F-:<10pt>]{Y} \ar[ll] \ar[u] & *+[F-:<10pt>]{X}  & & *+[F-:<10pt>]{Y} \ar[ll] \ar[u] \\
}
$$
\vskip -0.15in
  \caption{Confounding as a consequence of selection biases.}
  \label{fig:selection.mechanisms}
\end{wrapfigure}
Note that in the context of anticausal prediction tasks, the non-representativeness of the development data often arises due to selection mechanisms operating during the data collection phase. As illustrated in Figure \ref{fig:selection.mechanisms}a, confounding can be generated by selection mechanisms alone\footnote{Here, $S$ represents a binary variable which indicates whether the sample was included or not in the dataset, and the square frame around $S$ indicates that our dataset is generated conditional on $S$ being set to 1. Note that conditional on $S = 1$, we have that the path $X \leftarrow A \rightarrow S \leftarrow Y$ is open, since $S$ is a collider. This shows that $A$ satisfies the definition of a confounder.}. Furthermore, even when the confounder has stable causal effects on $X$ and on $Y$ (as represented by $\beta_{XA}$ and $\beta_{YA}$ in Figure \ref{fig:selection.mechanisms}b), selection mechanisms can still contribute to the association between $A$ and $Y$, making the data non-representative relative to target populations where this association is shifted. In general, selection mechanisms represent a common source of dataset shift in the joint distribution of the confounders and outcome variables, $P(A, Y)$, in anticausal prediction tasks.

\subsection{Linear residualization adjustment}

The linear residualization approach is implemented by regressing each separate input variable $X_j$ on the confounders, and then using the residuals of the linear regression fits as the new inputs for machine learning. Since the output variable is not included as a covariate in the regression fits, we have that the approach is actually based on the reduced model obtained by replacing eq. (\ref{eq:Y.model}) on eq. (\ref{eq:X.model}),
\begin{align}
\bfmX &= \bfGamma_{XA} \, \bfmA  + \bfGamma_{XY} \, Y + \bfmW_X \nonumber \\
&= \bfGamma_{XA} \, \bfmA + \bfGamma_{XY} \, (\bfGamma_{YA} \, \bfmA + W_Y) + \bfmW_X \nonumber \\
&= \bfOmega_{XA} \, \bfmA + \bfmW_X^\star
\end{align}
where $\bfOmega_{XA} = \bfGamma_{XA} + \bfGamma_{XY} \, \bfGamma_{YA}$, and $\bfmW_X^\star = \bfmW_X + \bfGamma_{XY} \, W_Y$. In practice, the residualized inputs, $\hat{\bfmX}_r$, are estimated as,
\begin{equation}
\hat{\bfmX}_r = \bfmX - \hat{\bfOmega}_{XA} \, \bfmA~,
\end{equation}
by regressing the catenated train and test set inputs on the catenated train and test confounder data in order to estimate $\bfOmega_{XA}$. Note that $\hat{\bfmX}_r$ corresponds to the estimated residual error term $\hat{\bfmW}_X^\star$.

\subsection{Causality-aware counterfactual adjustment}

Causality-aware counterfactual confounding adjustment is a special case of the causality-aware predictive modeling framework proposed by~\cite{achaibubneto2020a}. In the context of anticausal prediction tasks plagued by confounders, the key idea is to train and evaluate supervised ML algorithms on counterfactually simulated data which retains only the associations generated by the causal influences of the output variable on the inputs. The approach is implemented using a modification of Pearl's three step approach for the computation of deterministic counterfactuals~\cite{pearl2009,pearl2016}, where we regress the inputs on the confounders and output variable in order to estimate the model residuals and regression coefficients, and then simulate counterfactual data by adding back the model residuals to a linear predictor that no longer contains the confounder variables. Mechanistically, the causality-aware inputs are calculated as follows:
\begin{enumerate}[leftmargin=*]
\item Using the training set, estimate regression coefficients and residuals from the linear regression model,
$\bfmX^{tr} = \bfGamma_{XA}^{tr} \, \bfmA^{tr}  + \bfGamma_{XY}^{tr} \, Y^{tr} + \bfmW_X^{tr}$,
and then compute the respective counterfactual inputs, $\bfmX_c^{tr}$, by adding back the estimated residuals, $\hat{\bfmW}_X^{tr} = \bfmX^{tr} - \hat{\bfGamma}_{XA}^{tr} \, \bfmA^{tr} - \hat{\bfGamma}_{XY}^{tr} \, Y^{tr}$, to the quantity $\hat{\bfGamma}_{XY}^{tr} \, Y^{tr}$ (which represents the linear predictor obtained by excluding the confounder variables). That is, estimate the counterfactual features as,
\begin{equation}
\hat{\bfmX}_c^{tr} = \hat{\bfGamma}_{XY}^{tr} \, Y^{tr} + \hat{\bfmW}_X^{tr}~.
\label{eq:training.set.inputs.ca}
\end{equation}
\item Using the test set, compute the counterfactual inputs,
\begin{equation}
\hat{\bfmX}_c^{ts} = \bfmX^{ts} - \hat{\bfGamma}_{XA}^{tr} \, \bfmA^{ts}~,
\label{eq:test.set.inputs.ca}
\end{equation}
using the regression coefficients $\hat{\bfGamma}_{XA}^{tr}$ estimated in the training set.
\end{enumerate}
Once the training and test set counterfactual inputs, $\hat{\bfmX}_c^{tr}$ and $\hat{\bfmX}_c^{ts}$, have been generated we can then use $\hat{\bfmX}_c^{tr}$ and $Y^{tr}$ to train a linear learner, and then use $\hat{\bfmX}^{ts}_c$ to generate predictions that are free from the influence, or at least impacted by a lesser degree, by the observed confounders. Observe that the calculation of the test set causality-aware inputs in eq. (\ref{eq:test.set.inputs.ca}) does not uses the test set output, $Y^{ts}$. Observe, as well, that for large sample sizes, we have that the computation of the test set inputs using eq. (\ref{eq:test.set.inputs.ca}) is equivalent to computing the test set inputs using $\hat{\bfmX}_c^{ts} = \hat{\bfGamma}_{XY}^{ts} \, Y^{ts} + \hat{\bfmW}_X^{ts}$ since for large enough sample sizes we have that $\hat{\bfGamma}_{XA}^{tr} \approx \hat{\bfGamma}_{XA}^{ts}$ (assuming that the effects are stable across the training and test data) so that,
\begin{align}
\hat{\bfmX}_c^{ts} &= \bfmX^{ts} - \hat{\bfGamma}_{XA}^{tr} \, \bfmA^{ts} \approx \bfmX^{ts} - \hat{\bfGamma}_{XA}^{ts} \, \bfmA^{ts} = \hat{\bfGamma}_{XY}^{ts} \, Y^{ts} + \hat{\bfmW}_X^{ts}~.
\end{align}

\section{Results}

Before we present the main theoretical result of the paper, we first present the following result.

\begin{theorem}
For an anticausal prediction task influenced by a set of confounders $\bfmA$, the cross-covariance between the output variable, $Y$, and the inputs, $\bfmX$, is given by,
\begin{equation}
Cov(\bfmX, Y) \, = \, \bfGamma_{XY} + \bfGamma_{XA} \, Cov(\bfmA) \, \bfGamma_{YA}^T~,
\label{eq:cov.X.Y.main}
\end{equation}
while the asymptotic cross-covariances between $Y$ and the counterfactual inputs, $\bfmX_c$, and between $Y$ and the residualized inputs, $\bfmX_r$, are given respectively by,
\begin{align}
Cov(\bfmX_c, Y) \, &= \, \bfGamma_{XY}~, \label{eq:cov.Xc.Y.main} \\
Cov(\bfmX_r, Y) \, &= \, \bfGamma_{XY} (1 - \bfGamma_{YA} \, Cov(\bfmA) \, \bfGamma_{YA}^T)~. \label{eq:cov.Xr.Y.main}
\end{align}
\end{theorem}

\begin{theorem}
Under the conditions of Theorem 1, for each element $j$ of the vectors $Cov(\bfmX_c, Y)$ and $Cov(\bfmX_r, Y)$, we have that $|Cov(X_{c,j}, Y)| \geq |Cov(X_{r,j}, Y)|$.
\end{theorem}

The proofs of Theorems 1 and 2 are presented in Supplementary Section 1. In the special case of a single confounder variable $A$, equations (\ref{eq:cov.Xc.Y.main}) and (\ref{eq:cov.Xr.Y.main}) in Theorem 1 reduce to, $Cov(X_c, Y) = \gamma_{_{XY}}$ and $Cov(X_r, Y) = \gamma_{_{XY}} (1 - \gamma_{_{YA}}^2)$, and the result in Theorem 2 follows from,
$$
|Cov(X_c, Y)| = |\gamma_{_{XY}}| \; \geq \; |\gamma_{_{XY}}| (1 - \gamma_{_{YA}}^2) = |Cov(X_r, Y)|~,
$$
since $(1 - \gamma_{_{YA}}^2) \leq 1$ because $\gamma_{_{YA}}$ corresponds to the correlation between the $Y$ and $A$ variables\footnote{Direct application of Wright's method of path coefficients~\cite{wright1934} to the causal diagram $\xymatrix@-1.0pc{A \ar[r] \ar@/^0.5pc/[rr] & X & Y \ar[l]}$, shows that the marginal correlations among these three variables can be decomposed as $Cor(A, Y) = \gamma_{_{YA}}$, $Cor(A, X) = \gamma_{_{XA}} + \gamma_{_{YA}} \, \gamma_{_{XY}}$, and $Cor(X, Y) = \gamma_{_{XY}} + \gamma_{_{XA}} \, \gamma_{_{YA}}$ in terms of path coefficients.} and can only assume values between -1 and 1.

Observe, however, that while $|Cov(X_{c,j}, Y)| \geq |Cov(X_{r,j}, Y)|$, we have that $|Cov(X_{c,j}, Y)|$ might be greater or lesser than $|Cov(X, Y)| = |\gamma_{_{XY}} + \gamma_{_{XA}} \, \gamma_{_{YA}}|$ depending on whether $|\gamma_{_{XY}}|$ is greater or smaller than $|\gamma_{_{XY}} + \gamma_{_{XA}} \, \gamma_{_{YA}}|$. For instance, $|Cov(X_c, Y)| < |Cov(X, Y)|$ in situations where the spurious association contributed by the confounder (namely, $\gamma_{_{XA}} \, \gamma_{_{YA}}$) increases the strength of the association between $X$ and $Y$ (e.g., when $\gamma_{_{XY}} > 0$, $\gamma_{_{XA}} > 0$, and $\gamma_{_{YA}} > 0$), while $|Cov(X_c, Y)| > |Cov(X, Y)|$ in situations where $\gamma_{_{XA}} \, \gamma_{_{YA}}$ decreases the strength of the association (e.g., when $\gamma_{_{XY}} > 0$ but $\gamma_{_{XA}} \, \gamma_{_{YA}} < 0$ and $\gamma_{_{XY}} > |\gamma_{_{XA}} \, \gamma_{_{YA}}|$).

Now, observe that under the assumption of the absence of dataset shift~\cite{quionero2009} between training and test sets, we conjecture that the above result implies that the asymptotic predictive performance of linear learners trained with causality-aware counterfactual inputs outperforms the performance of linear learners trained on residualized inputs. Intuitively, this appears to be the case, since from Theorem 2 we have that for each input $X_j$, the linear association between the counterfactual input, $X_{c,j}$, and $Y$ is always stronger or equal than the linear association between the residual input, $X_{r,j}$, and $Y$. Since linear learners are only able to leverage linear associations between the inputs and the output for the predictions, it seems reasonable to expect that a linear learner trained with causality-aware counterfactual inputs will likely outperform the respective learner trained on the residualized inputs. In Supplementary Section 2 we actually prove (for a couple of special cases) that the expected mean squared error (MSE) for linear learners trained with a counterfactual feature is always smaller or equal than the expected MSE of models trained with the residualized feature, when sample size goes to infinity. (In the context of classification tasks, we conjecture that analogous results hold for performance metrics such as classification accuracy and area under the receiver operating characteristic curve.)

\section{Synthetic data illustrations}

Here, we present synthetic data illustrations of the points in the previous section for both regression and classification tasks. We evaluate predictive performance using mean squared error (MSE) in the regression task experiments, and accuracy (ACC) in the classification task experiments. Our experiments address both correctly specified and mispecified models, involving 2 input variables ($X_1$ and $X_2$) and 2 confounders. (See Supplementary Section 3 for a detailed description of the experiments). Figures \ref{fig:regression.task.experiments} and \ref{fig:classification.task.experiments} report the results.

\begin{figure}[!h]
\includegraphics[width=\linewidth]{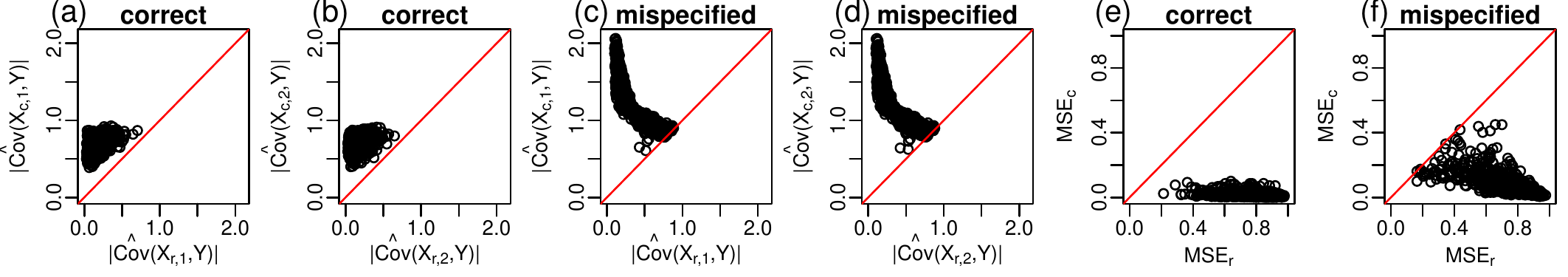}
\vskip -0.1in
\caption{Regression task experiments.}
\label{fig:regression.task.experiments}
\end{figure}

\begin{figure}[!h]
\includegraphics[width=\linewidth]{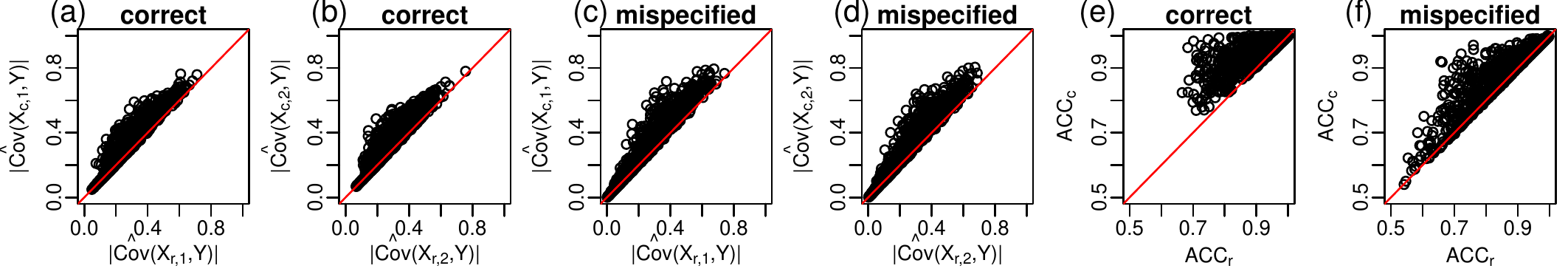}
\vskip -0.1in
\caption{Classification task experiments.}
\label{fig:classification.task.experiments}
\vskip -0.1in
\end{figure}

Panels a and b of Figures \ref{fig:regression.task.experiments} and \ref{fig:classification.task.experiments} illustrate the result from Theorem 2, showing that $|Cov(X_{c,j}, Y)| \geq |Cov(X_{r,j}, Y)|$ for both input variables $X_1$ and $X_2$, while panels c and d illustrate that the results still hold under model mispecification. Panels e and f in Figure \ref{fig:regression.task.experiments} show that $\mbox{MSE}_c \leq \mbox{MSE}_r$ for all simulations. Panels e and f in Figure \ref{fig:classification.task.experiments} show that $\mbox{ACC}_c \geq \mbox{ACC}_r$ for all simulations.

\section{Assessing the effectiveness of the confounding adjustment}

Our results suggest that, for linear learners, the causality-aware approach outperforms linear residualization, even when the true data-generation process does not follow the linear regression models adopted to process the inputs. It is important to keep in mind, however, that the causality-aware approach (as well as, the residualization) might fail to fully remove confounding from the predictions when the regression models are mispecified.

Here, we apply the methodology described in~\cite{chaibubneto2019} to evaluate the effectiveness of the causality-aware approach in our experiments. The approach is described in more detail in Supplementary Section 4, but the key idea is that, in anticausal prediction tasks, if the causality-aware adjustment was effective, then the only conditional independence relationship among the $Y$, $A$, and $\hat{Y}_c$ test set variables is given by $\hat{Y}_c^{ts} \ci A^{ts} \mid Y^{ts}$ (i.e., the prediction $\hat{Y}_c^{ts}$ is independent of the confounder given the outcome in the test set data).

Figure \ref{fig:ci.tests} reports the correlations and partial correlations between $\hat{Y}_c^{ts}$, $Y^{ts}$, and $A_1^{ts}$ for the regression experiments presented before. Panel a reports the results for the correctly specified models. The fact that the distribution of $\hat{cor}(\hat{Y}_c, A \mid Y)$ is tightly concentrated around 0, shows that the predictions are unconfounded. Panel b reports the results for the mispecified models and shows that $A_1$ is still confounding the predictions in this case (note the large spread of the $\hat{cor}(\hat{Y}_c, A \mid Y)$ distribution). These results illustrate that in many practical applications, the adoption of simple linear regression models for confounding adjustment might lead to unreliable inferences, and point to the need for more flexible models. To this end, we repeated these experiments replacing the linear models by more flexible additive models~\cite{hastie1990}, which are better able to capture non-linearities in the data (see Supplementary Section 5 for details on the additive-model based adjustments). Panels c and d report the results and show that the additive models effectively removed confounding in both the correctly specified (panel c) and mispecified (panel d) cases. (In particular, note how the adoption of the additive-models reduced the spread of the $\hat{cor}(\hat{Y}_c, A \mid Y)$ distribution in panel d compared to panel b.) Finally, panel e compares the MSE obtained with linear-regression vs additive-model adjustments. Note that for the correctly specified models, the MSE distributions are similar since the additive-model adapts to the data and ``mimics" a linear model in this case. For the mispecified models, on the other hand, we see higher MSE scores for the additive-model based adjustment because it effectively removed the confounder contribution from the predictive performance (recall that in our simulations the effects $\beta_{{X_j}{Y}}$, $\beta_{{Y}{A_j}}$, and $\beta_{{X_j}{A_j}}$ were positive, so that confounding increased the association between the inputs and output in our simulations and improved predictive performance.)\footnote{At this point it is important to clarify that one may argue that one should refrain from performing confounding adjustment whenever: (i) confounding is stable (i.e., the confounder effects on the inputs and outcome variables do not change between the training and test environments/target populations); and (ii) the confounder improves the predictive performance (by increasing the association between the inputs and output). While this is (strictly speaking) true, confounding generated by selection biases tend to be unstable and change across different environments/test sets. As described before, in this paper we assume that that a researcher has access to independent and identically distributed (i.i.d.) training and test sets derived from a developement population that is not representative of the target populations where the ML algorithm will be deployed. In this context, one is usually willing the perform confounding adjustment even when it decreases predictive performance.} Finally, Figure S4 reports scatterplots of the additive-model residualization vs the additive-model causality-aware approaches and provides empirical evidence in favor of the causality-aware adjustment.

\begin{figure*}[t]
\centerline{\includegraphics[width=\linewidth]{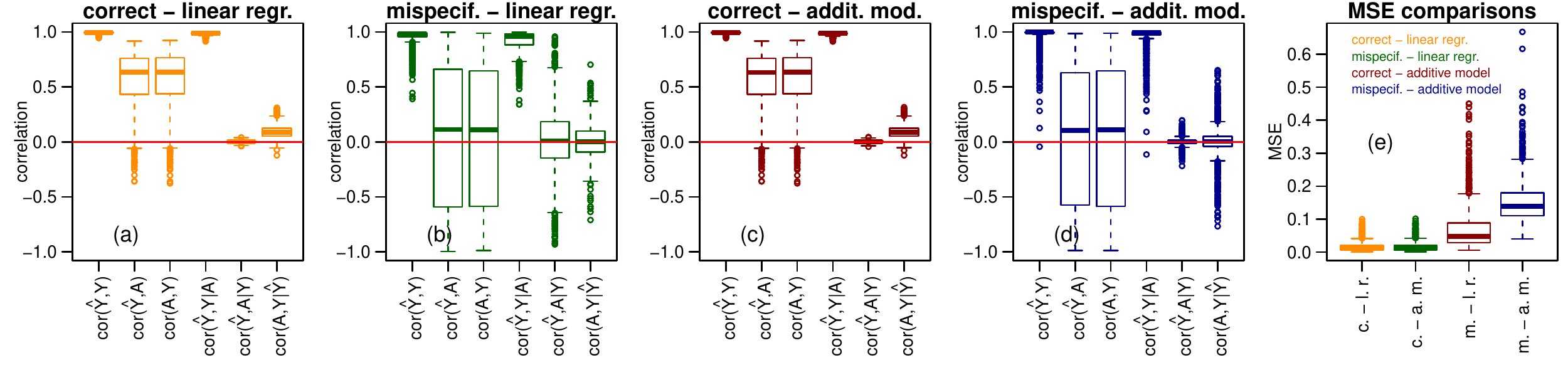}}
\vskip -0.1in
\caption{Assessment of the adjustment effectiveness.}
\vskip -0.2in
\label{fig:ci.tests}
\end{figure*}

\section{On the instability of the residualization approach}

So far, our analytical results and illustrations relied on the assumption of no dataset shift between the training and test sets. In this section, we illustrate how the causality-aware approach is more stable than residualization under dataset shifts of the joint distribution of the confounders and outcome variables. The stability properties of the causality-aware approach are presented in~\cite{achaibubneto2020a} where it was shown that because the expected MSE of the causality-aware approach does not depend on $A^{ts}$, the approach is stable w.r.t. shifts in $Cov(A^{ts}, Y^{ts})$. Here, we show that this is not the case for the residualization approach. For simplicity, we present the result for the toy model in Figure \ref{fig:shift.example}, where the double arrow represents an association generated by a selection mechanism, and where $Cov(A, Y) = \sigma_{AY}$, $Var(A) = \sigma_{AA}$, and $Var(Y) = \sigma_{YY}$. As shown in Supplementary Section 6.1, direct calculation of the expected MSE for the residualization approach shows that,
\begin{align*}
E[MSE_r] & = Var(Y^{ts}) + (\hat{\beta}_{r}^{tr})^2 Var(X_r^{ts}) - 2 \hat{\beta}_{r}^{tr} Cov(X_r^{ts}, Y^{ts})~, \hspace{0.3cm} \mbox{where,}
\end{align*}
\begin{align}
&Var(X_{r}^{ts}) = \sigma^2_X + \beta_{XY}^2 \, \sigma_{YY}^{ts} - (\beta_{XY}^2 \, (\sigma_{AY}^{ts})^2)/\sigma_{AA}^{ts}~, \\
&Cov(X_{r}^{ts}, Y^{ts}) = \beta_{XY} \, \sigma_{YY}^{ts} - (\beta_{XY} \, (\sigma_{AY}^{ts})^2)/\sigma_{AA}^{ts}~,
\end{align}
are still functions of $\sigma_{AY}^{ts}$. This shows that the expected MSE of the residualization approach will be unstable w.r.t. shifts in $Cov(A^{ts}, Y^{ts})$. Supplementary Section 6.2 shows that this result holds in general for linear structural causal models.

\subsection{Dataset shift experiments}

In our experiments, we generated dataset shift in $P(A, Y)$ by varying $Cov(A, Y) = \sigma_{AY}$, $Var(A) = \sigma_{AA}$, and $Var(Y) = \sigma_{YY}$ between the training and test sets. We, nonetheless, use the same values of $\beta_{XA}$, $\beta_{XY}$, and $\sigma^2_X$ in the generation of the training and test features, so that only the joint distribution $P(A, Y)$ differs between the training and test sets (while $P(X \mid A, Y)$ is stable).
\begin{wrapfigure}{r}{0.22\textwidth}
$$
\xymatrix@-1.2pc{
& *+[F-:<10pt>]{A} \ar[dl]_{\beta_{XA}} \ar@/^1.0pc/@{<->}[dr]^{\sigma_{AY}} &  \\
*+[F-:<10pt>]{X} && *+[F-:<10pt>]{Y} \ar[ll]^{\beta_{XY}} \\
}
$$
  \caption{}
  \label{fig:shift.example}
  \vskip -0.1in
\end{wrapfigure}

In our first experiment we generate 9 distinct test sets using different values of $\sigma_{AY}$ and $\sigma_{AA}$ relative to the training set, but where $\sigma_{YY}$ was still the same. (See Supplementary Section 6.3 for a detailed description of the experiments). Figure \ref{fig:dataset.shift.1} reports the results and clearly shows that while the predictive performance of the causality-aware approach was stable across the test sets the residualization approach was fairly unstable. Panel b shows the results of the first 3 simulations in more detail. Each line presents the MSE of the same trained model across the 9 distinct test sets, showing that the residualization results (red lines) vary widely across the test sets, while the causality-aware (blue lines) are fairly stable. Panel c reports the distributions of stability error (i.e., the standard deviation of the MSE scores across the 9 test sets) for both approaches.
\begin{figure}[!h]
\centerline{\includegraphics[width=\linewidth]{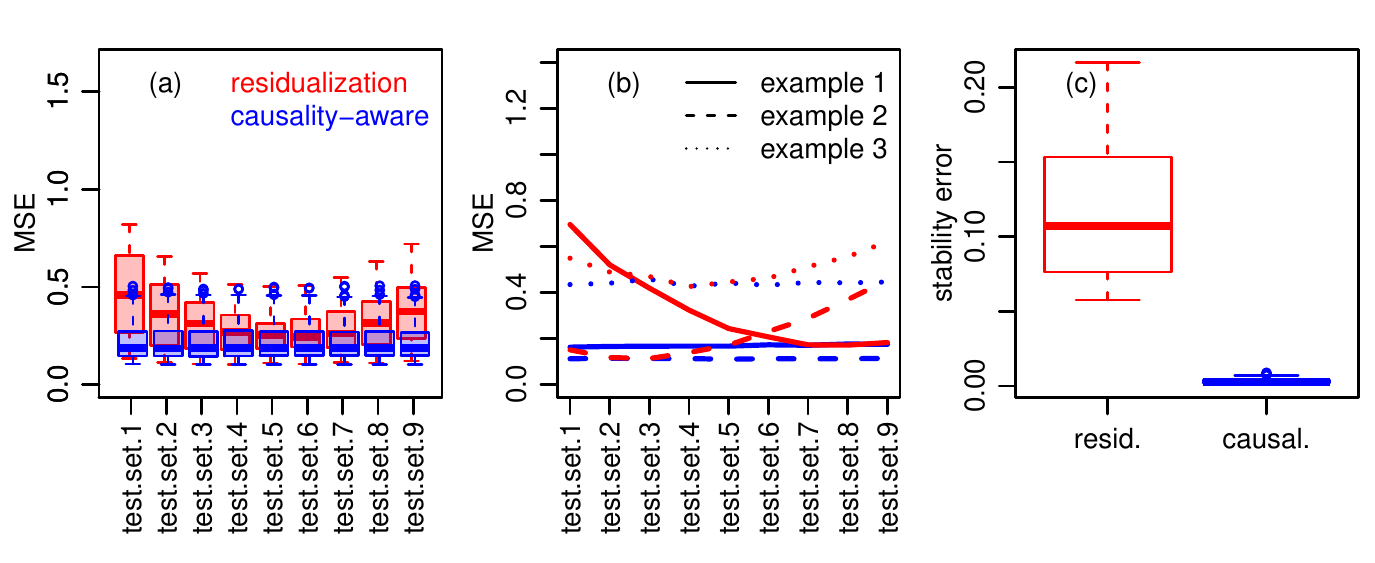}}
\vskip -0.1in
\caption{Stability illustrations, with fixed $Var(Y^{ts})$.}
\label{fig:dataset.shift.1}
\end{figure}

Observe, nonetheless, that because the expected MSE of any approach will, in general, depend on the variance of $Y^{ts}$ we performed an additional simulation study (Supplementary Figure S6) where we varied $Var(Y^{ts})$ from 1 to 3 across the 9 test sets. The results show that, while $MSE_c$ also changed across the test sets, the causality-aware approach is still much more stable than residualization.

These results suggest that the causality-aware approach still outperforms the residualization approach in the presence of dataset shift in the joint distribution of $P(A, Y)$. Both approaches, however, assume that the distribution $P(X \mid A, Y)$ is stable across the training and test sets. We point out, nonetheless, that in some important domains such as in diagnostic health applications dataset shifts on $P(A, Y)$ tend to be much more common than dataset shifts on $P(X \mid A, Y)$\footnote{For concreteness, suppose the goal is to classify mild vs severe cases of a given disease using the disease symptoms as inputs. Assume that gender is a confounder and suppose that gender is positively associated with the severe cases in the training data. This classifier will face dataset shift in $P(A, Y)$ whenever it is applied in a population with a different degree of association between gender and the disease labels than the association observed on the training data. This is arguably a very common situation given that, in practice, due to selection biases during data collection, ML algorithms are often trained on non-representative data. On the other hand, we will only observe dataset shift in $P(X \mid A, Y)$ in situations where there are physiological differences in the symptoms driven by gender and disease status in the individuals in the training set relative to the individuals in the distinct populations where the classifier will be deployed.}.

\section{Final remarks}

In this paper, we compare linear residualization against the causality-aware confounding adjustment. Our results suggest that the causality-aware approach outperforms residualization even when the regression models are mispecified. In this situation, however, the causality-aware approach might fail to fully remove the influence of the confounders from the predictions, and more flexible modeling approaches are needed. To this end, we describe how additive-models can help when linear regression fails. Furthermore, empirical comparisons between additive-model residualization and additive-model causality-aware adjustments (see Figure S4) and our stability results still favor the causality-aware approach. Taken together, these observations suggest that linear residualization can be safely replaced by causality-aware confounding adjustment in ML applications based on linear learners.

For non-linear learners, however, more research is needed, and we leave non-linear extensions of our results for future work. We point out, however, that the causality-aware approach can still be sometimes used to deconfound non-linear learners. As described in detail in~\cite{achaibubneto2020b}, we can use standard linear models to deconfound the feature representations learned by deep neural network (DNN) models. The key idea is that by training a highly accurate DNN using softmax activation at the classification layer, we have that, by construction, the feature representation learned by the last layer prior to the output layer will fit well a logistic regression model (since the softmax classification used to generate the output of the DNN is essentially performing logistic regression classification). This observation opens up the applicability of the causality-aware approach to a widely used class of non-linear learners.

The present work has focused on anticausal prediction tasks, as neuroimage and other health related prediction applications are often anticausal in nature. We leave the investigation of causal prediction tasks (where the inputs have a causal effect on the output) for future work.

\clearpage

\beginsupplement

\setcounter{section}{0}
\setcounter{theorem}{0}

\noindent {\huge SUPPLEMENT}

\section{Proofs of Theorems 1 and 2}

For the proof of Theorem 1 we will use the following properties of the cross-covariance operator\footnote{The cross-covariance, $Cov(\bfmA, \bfmB)$, between two vectors of random variables $\bfmA = (A_1, \ldots, A_{N_A})^T$ and $\bfmB = (B_1, \ldots, B_{N_B})^T$ is given by the $N_A \times N_B$ matrix with elements $Cov(A_i, B_j)$.}:
\begin{enumerate}
\item $Cov(\bfmZ_1 + \bfmZ_2, \bfmZ_3) = Cov(\bfmZ_1, \bfmZ_3) + Cov(\bfmZ_2, \bfmZ_3)$,
\item $Cov(\bfmB_1 \, \bfmZ_1, \bfmB_2 \, \bfmZ_2) = \bfmB_1 \, Cov(\bfmZ_1, \bfmZ_2) \, \bfmB_2^T$, where $\bfmB_1$ and $\bfmB_2$ are constant matrices
\item $Cov(\bfmZ, \bfmZ) = Cov(\bfmZ)$, where $Cov(\bfmZ)$ is the variance covariance matrix of $\bfmZ$.
\end{enumerate}

\subsection{Proof of Theorem 1}

For convenience, we reproduce Theorem 1 below.

\begin{theorem}
For an anticausal prediction task influenced by a set of confounders $\bfmA$, the cross-covariance between the output variable, $Y$, and the inputs, $\bfmX$, is given by,
\begin{equation}
Cov(\bfmX, Y) \, = \, \bfGamma_{XY} + \bfGamma_{XA} \, Cov(\bfmA) \, \bfGamma_{YA}^T~,
\label{eq:cov.X.Y}
\end{equation}
while the asymptotic cross-covariances between $Y$ and the counterfactual inputs, $\bfmX_c$, and between $Y$ and the residualized inputs, $\bfmX_r$, are given respectively by,
\begin{align}
Cov(\bfmX_c, Y) \, &= \, \bfGamma_{XY}~, \label{eq:cov.Xc.Y} \\
Cov(\bfmX_r, Y) \, &= \, \bfGamma_{XY} (1 - \bfGamma_{YA} \, Cov(\bfmA) \, \bfGamma_{YA}^T)~. \label{eq:cov.Xr.Y}
\end{align}
\end{theorem}

\begin{proof}

We first derive the result in equation (\ref{eq:cov.X.Y}). Direct computation shows that,
\begin{align*}
Cov(\bfmX, Y) &= Cov(\bfGamma_{XY} \, Y + \bfGamma_{XA} \, \bfmA + \bfmW_X, Y) \\
&= \bfGamma_{XY} \, Cov(Y, Y) + \bfGamma_{XA} \, Cov(\bfmA, Y) \\
&= \bfGamma_{XY} + \bfGamma_{XA} \, Cov(\bfmA, \, \bfGamma_{YA} \, \bfmA + W_Y) \\
&= \bfGamma_{XY} + \bfGamma_{XA} \, Cov(\bfmA, \, \bfmA) \, \bfGamma_{YA}^T \\
&= \bfGamma_{XY} + \bfGamma_{XA} \, Cov(\bfmA) \, \bfGamma_{YA}^T~,
\end{align*}
where the the first equality follows from the fact that $\bfmX = \bfGamma_{XY} \, Y + \bfGamma_{XA} \, \bfmA + \bfmW_X$; the second equality from the fact that $Cov(\bfmW_X, Y) = {\bf 0}$; the third equality from $Cov(Y,Y) = Var(Y) = 1$, and $Y = \bfGamma_{YA} \, \bfmA + W_Y$; the fourth equality from $Cov(\bfmA, W_Y) = {\bf 0}$; and the fifty equality from $Cov(\bfmA, \, \bfmA) = Cov(\bfmA)$.

Now, consider the causality-aware counterfactual input case. As described in the main text, the training set counterfactual inputs are computed as $\hat{\bfmX}_c^{tr} = \hat{\bfGamma}_{XY}^{tr} \, Y^{tr} + \hat{\bfmW}_X^{tr}$, where $\hat{\bfmW}_X^{tr} = \bfmX^{tr} - \hat{\bfGamma}_{XY}^{tr} \, Y^{tr} - \hat{\bfGamma}_{XA}^{tr} \, \bfmA^{tr}$, whereas the test set counterfactual inputs are computed as $\hat{\bfmX}_c^{ts} = \bfmX^{ts} - \hat{\bfGamma}_{XA}^{tr} \, \bfmA^{ts}$. Note that, as the sample size of the training set increases to infinity, we have that $\hat{\bfGamma}_{XY}^{tr}$ and $\hat{\bfGamma}_{XA}^{tr}$ converge, respectively, to $\bfGamma_{XY}$ and $\bfGamma_{XA}$. This implies that $\hat{\bfmX}_c^{tr}$ converges to $\bfmX_c = \bfGamma_{XY} \, Y + \bfmW_X$, where $\bfmW_X = \bfmX - \bfGamma_{XY} \, Y - \bfGamma_{XA} \, \bfmA$. Now, assuming that the joint distribution of the inputs and confounders is the same in the training and test set, we have that $\hat{\bfmX}_c^{ts}$ will also converge to $\bfmX_c = \bfGamma_{XY} \, Y + \bfmW_X$, since $\hat{\bfmX}_c^{ts}$ converges to $\bfmX - \bfGamma_{XA} \, \bfmA$ and,
$$
\bfmX - \bfGamma_{XA} \, \bfmA = \bfGamma_{XA} \, \bfmA  + \bfGamma_{XY} \, Y + \bfmW_X - \bfGamma_{XA} \, \bfmA = \bfGamma_{XY} \, Y + \bfmW_X~.
$$
Therefore, we have that direct computation of $Cov(\bfmX_c, Y)$ shows that,
\begin{align*}
Cov(\bfmX_c, Y) &= Cov(\bfGamma_{XY} \, Y + \bfmW_X, Y) \\
&= \bfGamma_{XY} \, Cov(Y, Y) \\
&= \bfGamma_{XY}~.
\end{align*}

Finally, consider the residual inputs, $\bfmX_r$. As described in the main text, $\bfmX_r$ is computed as $\hat{\bfmX}_r = \bfmX - \hat{\bfOmega}_{XA} \, \bfmA$. As sample size increases, $\hat{\bfmX}_r$ converges to $\bfmX_r = \bfmX - \bfOmega_{XA} \, \bfmA$, where $\bfOmega_{XA} = \bfGamma_{XA} + \bfGamma_{XY} \, \bfGamma_{YA}$. Direct computation of $Cov(\bfmX_r, Y)$ shows that,
\begin{align*}
Cov(\bfmX_r, Y) &= Cov(\bfmX - \bfOmega_{XA} \, \bfmA, Y) \\
&= Cov(\bfmX, Y) - \bfOmega_{XA} \, Cov(\bfmA, Y) \\
&= \bfGamma_{XY} + \bfGamma_{XA} \, Cov(\bfmA) \, \bfGamma_{YA}^T - \bfOmega_{XA} \, Cov(\bfmA, \, \bfGamma_{YA} \, \bfmA + W_Y) \\
&= \bfGamma_{XY} + \bfGamma_{XA} \, Cov(\bfmA) \, \bfGamma_{YA}^T - \bfOmega_{XA} \, Cov(\bfmA) \, \bfGamma_{YA}^T \\
&= \bfGamma_{XY} + \bfGamma_{XA} \, Cov(\bfmA) \, \bfGamma_{YA}^T - (\bfGamma_{XA} + \bfGamma_{XY} \, \bfGamma_{YA}) \, Cov(\bfmA) \, \bfGamma_{YA}^T \\
&= \bfGamma_{XY} - \bfGamma_{XY} \, \bfGamma_{YA} \, Cov(\bfmA) \, \bfGamma_{YA}^T \\
&= \bfGamma_{XY} (1 - \bfGamma_{YA} \, Cov(\bfmA) \, \bfGamma_{YA}^T)~.
\end{align*}
\end{proof}

\subsection{Proof of Theorem 2}

For the sake of clarity, we first prove Theorem 2 (reproduced below) in the special case where the confounder variables are uncorrelated. The general case of correlated confounders is proved in Section 1.2.2.

\begin{theorem}
Under the conditions of Theorem 1, for each element $j$ of the vectors $Cov(\bfmX_c, Y)$ and $Cov(\bfmX_r, Y)$, we have that $|Cov(X_{c,j}, Y)| \geq |Cov(X_{r,j}, Y)|$.
\end{theorem}

\subsubsection{The uncorrelated confounders case}

\begin{proof}
In the uncorrelated case, it follows that $Cov(\bfmA) = Cor(\bfmA) = \bfI_k$ since $Cov(A_i, A_i') = 0$ for $i \not= i'$, and $Cov(A_i, A_i) = Var(A_i) = 1$, for $i = 1, \ldots, k$, since $A_i$ is standardized. Hence, $Cov(\bfmX_r, Y) = \bfGamma_{XY} (1 - \bfGamma_{YA} \, \bfGamma_{YA}^T)$. Note, as well, that for the anticausal model in Figure 1 in the main text, we have that each entry $i$ of the $1 \times k$ matrix $\bfGamma_{YA}$ corresponds to $\gamma_{Y A_i} = Cov(Y, A_i) = Cor(Y, A_i)$, so that $\bfGamma_{YA} \, \bfGamma_{YA}^T = Cor(Y, \bfmA) \, Cor(Y, \bfmA)^T = \sum_{i=1}^k Cor(Y, A_i)^2$. Now, observe that $0 \leq \bfGamma_{YA} \, \bfGamma_{YA}^T \leq 1$, since it corresponds to the multiple correlation coefficient between $Y$ and $\bfmA$ in the special case where the confounders are uncorrelated. (Recall that, by definition, the multiple correlation coefficient is computed as $R^2_{Y\bfmA} = Cor(Y, \bfmA) \, Cor(\bfmA)^{-1} \, Cor(Y, \bfmA)^T$, and reduces to $Cor(Y, \bfmA) \, Cor(Y, \bfmA)^T$ in the special case where $Cor(\bfmA) = \bfI_k$.) Hence, it follows that $0 \leq (1 - \bfGamma_{YA} \, \bfGamma_{YA}^T) \leq 1$. Therefore, it follows that,
$$
|Cov(\bfmX_c, Y)| = |\bfGamma_{XY}| \, \geq \, |\bfGamma_{XY}| \, (1 - \bfGamma_{YA} \, \bfGamma_{YA}^T) = |Cov(\bfmX_r, Y)|~,
$$
or, equivalently, $|Cov(X_{c,j}, Y)| = |\gamma_{_{X_j,Y}}| \geq |\gamma_{_{X_j,Y}}| (1 - \sum_{i=1}^k Cor(Y, A_i)^2) = |Cov(X_{r,j}, Y)|$, for each input variable $j$.
\end{proof}

\subsubsection{The correlated confounders case}

The key idea to prove Theorem 2 in the general correlated confounders case is to recast the problem in terms of a singular value decomposition (SVD) of the confounder data. But first, we prove the following lemma that will be needed for the proof. (Note that for this proof, we also make a notational distinction between sets of random variables, and the respective data matrices. That is, we still represent sets of random variables in italic and boldface, e.g., $\bfmX = (X_1, \ldots, X_p)^T$ and $\bfmA = (A_1, \ldots, A_k)^T$, whereas data matrices are represented in boldface, namely, $\bfX$ is a $n \times p$ input data matrix, $\bfA$ is a $n \times k$ input data matrix, and $\bfY$ is a $n \times 1$ output data matrix.)

\begin{lemma}
Let $\bfA = \bfU \, \bfD \, \bfV^T$ represent a SVD of the confounder data matrix $\bfA$. By scaling $\bfU$ as $\tilde{\bfU} = (n-1)^{\frac{1}{2}} \, \bfU$ and $\bfV$ as $\tilde{\bfV} = (n-1)^{-\frac{1}{2}} \, \bfV$, we still obtain a valid SVD of $\bfA = \tilde{\bfU} \, \bfD \, \tilde{\bfV}^T$, but where $\tilde{\bfU}$ represents a scaled matrix, whose column vectors are uncorrelated and have variance exactly equal to 1.
\end{lemma}

\begin{proof}
\vskip -0.1in
Consider the singular value decomposition of confounder data matrix, $\bfA = \bfU \, \bfD \, \bfV^T$, where $\bfU$ is a $n \times k$ matrix of orthonormal eigenvectors of $\bfA \bfA^T$, $\bfD$ is a $k \times k$ diagonal matrix of singular values, and $\bfV$ is a $k \times k$ matrix of orthonormal eigenvectors of $\bfA^T \bfA$.

First, note that since $\bfA$ is a scaled matrix, it follows that the sample covariance of $\bfmA$ is given by,
\begin{align}
\hat{Cov(\bfmA)} &= (n-1)^{-1}\bfA^T \bfA = (n-1)^{-1} \bfV \, \bfD^T \, \bfU^T \, \bfU \, \bfD \, \bfV^T = (n-1)^{-1} \bfV \, \bfD^2 \, \bfV^T \nonumber \\
&= [(n-1)^{-\frac{1}{2}} \bfV] \, \bfD^2 \, [(n-1)^{-\frac{1}{2}} \bfV]^T~,
\label{eq:cov.A}
\end{align}
where the third equality follows from the fact $\bfU^T \bfU = \bfI_k$ since $\bfU$ is orthonormal.

Now, observe that while the SVD will produce a matrix $\bfU$ whose columns are orthogonal, the variance of each element of $\bfU$ will not be 1. Hence, we need to obtain a scaled version of the variable $\bfmU$, denoted $\tilde{\bfmU}$, such that $Cov(\tilde{\bfmU}) = \bfI_k$, that is, we need to find out a $\tilde{\bfmU}$ such that,
\begin{equation}
\hat{Cov}(\tilde{\bfmU}) = (n-1)^{-1}\tilde{\bfU}^T \tilde{\bfU} = \bfI_k.
\label{eq:eq5}
\end{equation}
Now, because by construction we have that $\bfU^T \bfU = \bfI_k$ it follows that by choosing $\tilde{\bfU} = (n-1)^{\frac{1}{2}} \bfU$ we have that eq. (\ref{eq:eq5}) is satisfied since,
\begin{equation}
\hat{Cov}(\tilde{\bfmU}) = (n-1)^{-1} ((n-1)^{\frac{1}{2}} \bfU)^T ((n-1)^{\frac{1}{2}} \bfU) = \bfU^T \bfU = \bfI_k~,
\end{equation}
so that $\tilde{\bfU}$ provides a scaled version of $\bfU$, whose columns are still orthogonal, but where the variance of the elements of each column of $\tilde{\bfU}$ is exactly 1. From the above, it follows that by rescaling $\bfU$ and $\bfV$ matrices as,
\begin{equation}
\tilde{\bfU} = (n-1)^{\frac{1}{2}} \, \bfU~, \hspace{0.5cm}
\tilde{\bfV} = (n-1)^{-\frac{1}{2}} \, \bfV~,
\end{equation}
we still obtain a valid singular value decomposition of $\bfA$, since,
$$
\tilde{\bfU} \, \bfD \, \tilde{\bfV}^T = [(n-1)^{\frac{1}{2}} \bfU] \, \bfD \, [(n-1)^{-\frac{1}{2}}\bfV]^T = \bfU \, \bfD \, \bfV^T = \bfA~.
$$
\end{proof}

We now prove Theorem 2 in the general case of correlated confounders.

\begin{proof}
It suffices to show that we can always reparameterize our linear models using the full rank (scaled) singular value decomposition (SVD) of the confounder data, $\bfA = \tilde{\bfU} \, \bfD \, \tilde{\bfV}^T$, presented in Lemma 1, where we replace the matrix versions of the original models,
\begin{align*}
\bfX &= \bfA \, \bfGamma_{XA}^T + \bfY \, \bfGamma_{XY}^T + \bfW_X~, \\
\bfY &= \bfA \, \bfGamma_{YA}^T + \bfW_Y~,
\end{align*}
where $\bfX$ and $\bfW_X$ have dimension $n \times p$, and $\bfA$ and $\bfY$ are $n \times k$ and $n \times 1$ matrices, respectively, by the SVD matrix regression models,
\begin{align*}
\bfX &= \tilde{\bfU} \, \bfGamma_{X\tilde{U}}^T + \bfY \, \bfGamma_{XY}^T + \bfW_X~, \\
\bfY &= \tilde{\bfU} \, \bfGamma_{Y\tilde{U}}^T + \bfW_Y~,
\end{align*}
where $\bfGamma_{X\tilde{U}} = \bfGamma_{XA} \, \tilde{\bfV} \, \bfD$ and $\bfGamma_{Y\tilde{U}} = \bfGamma_{YA} \, \tilde{\bfV} \, \bfD$, and where $\tilde{\bfU}$ is a $n \times k$ matrix whose $k$ column vectors are orthogonal to each other, and have variances exactly equal to 1. By adopting the above reparameterization, we effectively replace the original correlated confounder variables $A_j$ by the uncorrelated and scaled variables $\tilde{U}_j$ (which correspond to a linear combination of the $A_j$ variables).

Note that the residualized inputs computed in the original and parameterized models are exactly the same since,
\begin{align*}
\bfX_r &= \bfX - \bfA \, \bfOmega_{XA}^T \\
&= \bfX - \bfA \, (\bfGamma_{XA} + \bfGamma_{XY} \, \bfGamma_{YA})^T \\
&= \bfX - \bfA \, \bfGamma_{XA}^T - \bfA \bfGamma_{YA}^T \, \bfGamma_{XY}^T \\
&= \bfX - \tilde{\bfU} \, \bfD \, \tilde{\bfV}^T \, \bfGamma_{XA}^T - \tilde{\bfU} \, \bfD \, \tilde{\bfV}^T \bfGamma_{YA}^T \, \bfGamma_{XY}^T \\
&= \bfX - \tilde{\bfU} \, \bfGamma_{X\tilde{U}}^T - \tilde{\bfU} \bfGamma_{Y\tilde{U}}^T \, \bfGamma_{XY}^T \\
&= \bfX - \tilde{\bfU} \, (\bfGamma_{X\tilde{U}} + \bfGamma_{XY} \, \bfGamma_{Y\tilde{U}})^T \\
&= \bfX - \tilde{\bfU} \, \bfOmega_{X\tilde{U}}^T~,
\end{align*}
Similarly, the counterfactual inputs are the same since $\bfX_c = \bfY \, \bfGamma_{XY}^T + \bfW_X$ and,
\begin{align*}
\bfW_X &= \bfX - \bfY \, \bfGamma_{XY}^T - \bfA \, \bfGamma_{XA}^T \\
&= \bfX - \bfY \, \bfGamma_{XY}^T - \tilde{\bfU} \, \bfD \, \tilde{\bfV}^T \, \, \bfGamma_{XA}^T \\
&= \bfX - \bfY \, \bfGamma_{XY}^T - \tilde{\bfU} \, \bfGamma_{X\tilde{U}}^T~.
\end{align*}

Hence, by working with the reparameterized regression models,
\begin{align*}
\bfmX &= \bfGamma_{X\tilde{U}} \, \tilde{\bfmU}  + \bfGamma_{XY} \, Y + \bfmW_X~, \\
Y &= \bfGamma_{Y\tilde{U}} \, \tilde{\bfmU} + W_Y~,
\end{align*}
where $\tilde{\bfmU}$ represents the new uncorrelated confounder variables, we have from Theorem 1 that $Cov(\bfmX_r, Y) = \bfGamma_{XY} (1 - \bfGamma_{Y\tilde{U}} \, Cov(\tilde{\bfmU}) \, \bfGamma_{Y\tilde{U}}^T) = \bfGamma_{XY} (1 - \bfGamma_{Y\tilde{U}} \, \bfGamma_{Y\tilde{U}}^T)$, so that the result follows from the proof for the uncorrelated confounder case presented above.
\end{proof}

\section{Expected mean squared error analytical comparisons}

Here, we present an analytical comparison of the expected mean squared errors for the causality-aware approach, $E[MSE_c]$ against the residualization approach, $E[MSE_r]$. We show that under the conditions of Theorem 1, and assuming there is no dataset shift between the training and test sets, we have that $E[MSE_c] \leq E[MSE_r]$. We show this result for models with up to 2 features (as the algebra becomes too heavy for higher dimensions).

Before we show the results, we first re-express the expected MSE as a function of variances and covariances of the outcome and input variables.


Let $\hat{Y} = \bfmX_{ts} \hat{\bfBeta}^{tr}$ represent the prediction of a linear model, where $\bfmX_{ts}$ represents the test set features, and $\hat{\bfBeta}^{tr}$ represents the regression coefficients estimated from the training set. Without loss of generality assume that the data has been centered. By definition, the expected mean squared error of the prediction is given by,
\begin{align}
E[MSE] &= E[(Y_{ts} - \hat{Y})^2] = E[Y_{ts}^2] + E[\hat{Y}^2] - 2 E[\hat{Y} Y_{ts}]~, \\
&= Var(Y_{ts}) + E[\hat{Y}^2] - 2 Cov(\hat{Y}, Y_{ts}),
\end{align}
where the second equality follows from the fact that $E[Y_{ts}] = 0$. Now, as sample size goes to infinity we have that $\hat{Y}$ converges to $\bfmX_{ts} \bfBeta^{tr} = \sum_{j=1}^{p} X_{j,ts} \beta_j^{tr}$. Furthermore, assuming the absence of dataset shift between the training and test sets we have that,
\begin{align}
E[\hat{Y}^2]  &=  E[(\sum_{j=1}^{p} X_{j,ts} \beta_j^{tr})^2] = E[(\sum_{j=1}^{p} X_{j} \beta_j)^2] \\
Cov(\hat{Y}, Y_{ts}) &= Cov(\sum_{j=1}^{p} X_{j,ts} \beta_j^{tr}, Y_{ts}) = Cov(\sum_{j=1}^{p} X_{j} \beta_j, Y),
\end{align}
since any moments of $\bfmX$ and $Y$ will be the same for the training and test data, and $\beta_j^{tr} = \beta_j^{ts} = \beta_j$, so that we can drop the $tr$ and $ts$ (superscripts) subscripts from the notation. Therefore, we have that,
\begin{align}
E[MSE] &= Var(Y) + E[(\sum_{j=1}^{p} X_{j} \beta_j)^2] - 2 Cov(\sum_{j=1}^{p} X_{j} \beta_j, Y), \nonumber \\
&= Var(Y) + \sum_{j=1}^{p} \beta_j^2 E[X_{j}^2] + 2 \sum_{j < k} \beta_j \beta_k E[X_{j} X_{k}] - 2 \sum_{j=1}^{p} \beta_j Cov(X_{j}, Y) \nonumber \\
&= Var(Y) + \sum_{j=1}^{p} \beta_j^2 Var(X_{j}) + 2 \sum_{j < k} \beta_j \beta_k Cov(X_{j}, X_{k}) - 2 \sum_{j=1}^{p} \beta_j Cov(X_{j}, Y)~.
\end{align}

\subsection{Single feature case}

In the single feature case we have that $E[MSE]$ reduces to,
\begin{equation}
E[MSE] = Var(Y) + \beta^2 Var(X) - 2 \beta Cov(X, Y)~.
\end{equation}
Observe that the above quantity can be re-expressed as,
\begin{align}
E[MSE] &= Var(Y) + \frac{Cov(X, Y)^2}{Var(X)^2} Var(X) - 2 \frac{Cov(X, Y)}{Var(X)} Cov(X, Y) \nonumber \\
&= Var(Y) - \frac{Cov(X, Y)^2}{Var(X)}~, \label{eq:mse}
\end{align}
since $\beta$ represents the (asymptotic) coefficient of the regression of $Y$ on $X$ and is given by $Cov(X, Y)/Var(X)$.

In order to simplify notation we let,
$$
\xymatrix@-1.0pc{
& *+[F-:<10pt>]{A} \ar[dl]_{\theta} \ar[dr]^{\phi} &  \\
*+[F-:<10pt>]{X} & & *+[F-:<10pt>]{Y} \ar[ll]^{\gamma} \\
}
$$
and we assume that the data has been standardized so that $Var(Y) = Var(A) = Var(X) = 1$. The structural causal model is expressed as,
\begin{align}
A &= W_A~, \\
Y &= \phi \, A + W_Y~, \\
X &= \gamma \, Y + \theta \, A + W_X~, \hspace{0.3cm} Var(W_X) = \sigma^2~,
\end{align}
and the causality-aware and residualized features are expressed as,
\begin{align}
X_c &= X - \theta \, A = \gamma \, Y + W_X~, \\
X_r &= X - (\theta + \phi \, \gamma) A = X_c - \phi \, \gamma A~.
\end{align}

Direct computations show that,
\begin{align}
Cov(X_{c}, Y) &= Cov(\gamma \, Y + W_X, Y) = \gamma Var(Y)  \nonumber \\
&= \gamma~, \label{eq:cov.Xc.Y} \\
Var(X_{c}) &= Var(\gamma \, Y + W_X) = \gamma^2 Var(Y) + Var(W_X)  \nonumber \\
&= \sigma^2 + \gamma^2~, \label{eq:var.Xc}
\end{align}
and,
\begin{align}
Cov(X_r, Y) &= Cov(X_c - \phi \, \gamma A, Y) \nonumber \\
&= Cov(X_c, Y) - \phi \, \gamma Cov(A, Y) = \gamma - \phi \, \gamma \, \phi \nonumber \\
&= \gamma \, (1 - \phi^2)~, \label{eq:cov.Xr.Y}
\end{align}
\begin{align}
Var(X_r) &= Var(X_c - \phi \, \gamma A) \nonumber \\
&= Var(X_c) + \phi^2 \, \gamma^2 Var(A) - 2 \phi \, \gamma Cov(X_c, A) \nonumber \\
&= \sigma^2 + \gamma^2 + \phi^2 \, \gamma^2 - 2 \phi \gamma \, Cov(\gamma \, Y + W_X, A) \nonumber \\
&= \sigma^2 + \gamma^2 + \phi^2 \, \gamma^2 - 2 \phi \gamma \, \gamma \, \phi \nonumber \\
&= \sigma^2 + \gamma^2 \, (1 - \phi^2)~. \label{eq:var.Xr}
\end{align}

By replacing eq. (\ref{eq:cov.Xc.Y}) and (\ref{eq:var.Xc}) on eq. (\ref{eq:mse}) we have that,
\begin{equation}
E[MSE_c] = 1 - \frac{\gamma^2}{\sigma^2 + \gamma^2}~.
\end{equation}
Similarly, replacing eq. (\ref{eq:cov.Xr.Y}) and (\ref{eq:var.Xr}) on eq. (\ref{eq:mse}), shows that,
\begin{equation}
E[MSE_r] = 1 - \frac{\gamma^2 \, (1 - \phi^2)^2}{\sigma^2 + \gamma^2 \, (1 - \phi^2)}~.
\end{equation}

Now, observe that,
\begin{equation}
\frac{\gamma^2 \, (1 - \phi^2)^2}{\sigma^2 + \gamma^2 \, (1 - \phi^2)} = \frac{\gamma^2}{\frac{\sigma^2}{(1 - \phi^2)^2} + \frac{\gamma^2}{1 - \phi^2}} \leq \frac{\gamma^2}{\sigma^2 + \gamma^2}~,
\end{equation}
since $\phi = Cor(Y, A)$ implies that $0 \leq (1 - \phi^2) \leq 1$, so that $\sigma^2/(1 - \phi^2)^2 \geq \sigma^2$ and $\gamma^2/(1 - \phi^2) \geq \gamma^2$. Therefore, it follows that,
\begin{equation}
E[MSE_r] = 1 - \frac{\gamma^2 \, (1 - \phi^2)^2}{\sigma^2 + \gamma^2 \, (1 - \phi^2)} \, \geq \, 1 - \frac{\gamma^2}{\sigma^2 + \gamma^2} = E[MSE_c]~.
\end{equation}

\subsection{Two features case}

In the two features case we have that,
\begin{align}
E[MSE] &= Var(Y) + \beta_1^2 Var(X_1) + \beta_2^2 Var(X_2) + 2 \beta_1 \, \beta_2 Cov(X_1, X_2) - \nonumber \\
& -2 \beta_1 Cov(X_1, Y) - 2 \beta_1 Cov(X_2, Y)~.
\end{align}

Now, observe that,
\begin{align}
\beta_1 &= \frac{Cov(X_1, Y) Var(X_2) - Cov(X_2, Y) Cov(X_1, X_2)}{Var(X_1) Var(X_2) - Cov(X_1, X_2)^2}~, \\
\beta_2 &= \frac{Cov(X_2, Y) Var(X_1) - Cov(X_1, Y) Cov(X_1, X_2)}{Var(X_1) Var(X_2) - Cov(X_1, X_2)^2}~,
\end{align}
and we have after some algebraic manipulations that,
\begin{align}
E&[MSE] = Var(Y) + \nonumber \\
&+\frac{2 Cov(X_1, Y) Cov(X_2, Y) Cov(X_1, X_2) - Var(X_1) Cov(X_2, Y)^2 - Var(X_2) Cov(X_1, Y)^2}{Var(X_1) Var(X_2) - Cov(X_1, X_2)^2}~ \label{eq:mse.2.var}.
\end{align}

Again, we assume that the data has been standardized and follows the model,
$$
\xymatrix@-1.2pc{
&&& *+[F-:<10pt>]{A} \ar[dll]_{\theta_1} \ar[dddll]|-{\theta_2} \ar[ddr]^{\phi} &  \\
W_{X_1} \ar@/_1.0pc/@{<->}[dd] \ar[r] & *+[F-:<10pt>]{X_1} & &  \\
&&& & *+[F-:<10pt>]{Y} \ar@/_0.5pc/[lllu]|-{\gamma_1} \ar[llld]^{\gamma_2} \\
W_{X_2} \ar[r] & *+[F-:<10pt>]{X_2} & & & \\
}
$$
described by the equations,
\begin{align}
A &= W_A~, \\
Y &= \phi \, A + W_Y~, \\
X_1 &= \gamma_1 \, Y + \theta_1 \, A + W_{X_1}~, \\
X_2 &= \gamma_2 \, Y + \theta_2 \, A + W_{X_2}~,
\end{align}
where we assume that the correlated error terms have mean and covariance given by,
\begin{equation}
E
\begin{pmatrix}
W_{X_1} \\
W_{X_2} \\
\end{pmatrix}
=
\begin{pmatrix}
0 \\
0 \\
\end{pmatrix}~, \hspace{0.2cm}
\bfSigma_{\bfW} =
\begin{pmatrix}
Var(W_{X_1}) & Cov(W_{X_1}, W_{X_2}) \\
Cov(W_{X_1}, W_{X_2}) & Var(W_{X_2}) \\
\end{pmatrix}
=
\begin{pmatrix}
\sigma_{11} & \sigma_{12} \\
\sigma_{12} & \sigma_{22} \\
\end{pmatrix}~,
\end{equation}
and where the causality-aware and the residualized features are given by,
\begin{align}
X_{j,c} &= X_j - \theta_j \, A = \gamma_j \, Y + W_{X_j}~, \\
X_{j,r} &= X_j - (\theta_j + \phi \, \gamma_j) A = X_{j,c} - \phi \, \gamma_j A~.
\end{align}
for $j = 1, 2$. Direct calculation shows that,
\begin{align}
Var&(X_{j,c}) = \sigma_{jj} + \gamma_j^2~, \label{eq:stats.c.1} \\
Cov&(X_{j,c}, Y) = \gamma_j~, \\
Cov&(X_{1,c}, X_{2,c}) = \sigma_{12} + \gamma_1 \, \gamma_2~,  \label{eq:stats.c.3}
\end{align}
and that,
\begin{align}
Var&(X_{j,r}) = \sigma_{jj} + \gamma_j^2 \, (1 - \phi^2)~, \label{eq:stats.r.1} \\
Cov&(X_{j,r}, Y) = \gamma_j \, (1 - \phi^2)~, \\
Cov&(X_{1,r}, X_{2,r}) = \sigma_{12} + \gamma_1 \, \gamma_2 \, (1 - \phi^2)~.  \label{eq:stats.r.3}
\end{align}

Direct replacement of equations (\ref{eq:stats.c.1})-(\ref{eq:stats.c.3}) on equation (\ref{eq:mse.2.var}) shows that,
\begin{equation}
E[MSE_c] = 1 - \frac{\sigma_{11} \, \gamma_2^2 + \sigma_{22} \, \gamma_1^2 - 2 \, \gamma_1 \, \gamma_2 \, \sigma_{12}}{\sigma_{11} \, \sigma_{22} - \sigma_{12}^2 + (\sigma_{11} \, \gamma_2^2 + \sigma_{22} \, \gamma_1^2 - 2 \, \gamma_1 \, \gamma_2 \, \sigma_{12})}~,
\end{equation}
whereas replacement of equations (\ref{eq:stats.r.1})-(\ref{eq:stats.r.3}) on equation (\ref{eq:mse.2.var}) shows that,
\begin{equation}
E[MSE_r] = 1 - \frac{(\sigma_{11} \, \gamma_2^2 + \sigma_{22} \, \gamma_1^2 - 2 \, \gamma_1 \, \gamma_2 \, \sigma_{12})(1 - \phi^2)^2}{\sigma_{11} \, \sigma_{22} - \sigma_{12}^2 + (\sigma_{11} \, \gamma_2^2 + \sigma_{22} \, \gamma_1^2 - 2 \, \gamma_1 \, \gamma_2 \, \sigma_{12}) (1 - \phi^2)}~.
\end{equation}

Now, observe that,
\begin{align}
&\frac{(\sigma_{11} \, \gamma_2^2 + \sigma_{22} \, \gamma_1^2 - 2 \, \gamma_1 \, \gamma_2 \, \sigma_{12})(1 - \phi^2)^2}{\sigma_{11} \, \sigma_{22} - \sigma_{12}^2 + (\sigma_{11} \, \gamma_2^2 + \sigma_{22} \, \gamma_1^2 - 2 \, \gamma_1 \, \gamma_2 \, \sigma_{12}) (1 - \phi^2)} \\
&= \frac{(\sigma_{11} \, \gamma_2^2 + \sigma_{22} \, \gamma_1^2 - 2 \, \gamma_1 \, \gamma_2 \, \sigma_{12})}{\frac{\sigma_{11} \, \sigma_{22} - \sigma_{12}^2}{(1 - \phi^2)^2} + \frac{\sigma_{11} \, \gamma_2^2 + \sigma_{22} \, \gamma_1^2 - 2 \, \gamma_1 \, \gamma_2 \, \sigma_{12}}{1 - \phi^2}} \\
&\leq \frac{\sigma_{11} \, \gamma_2^2 + \sigma_{22} \, \gamma_1^2 - 2 \, \gamma_1 \, \gamma_2 \, \sigma_{12}}{\sigma_{11} \, \sigma_{22} - \sigma_{12}^2 + (\sigma_{11} \, \gamma_2^2 + \sigma_{22} \, \gamma_1^2 - 2 \, \gamma_1 \, \gamma_2 \, \sigma_{12})}
\end{align}
since $0 \leq (1 - \phi)^2 \leq 1$, $(\sigma_{11} \, \sigma_{22} - \sigma_{12}^2) > 0$ (as it corresponds to the determinant of the positive definite covariance matrix $\bfSigma_{\bfW}$), and $(\sigma_{11} \, \gamma_2^2 + \sigma_{22} \, \gamma_1^2 - 2 \, \gamma_1 \, \gamma_2 \, \sigma_{12}) > 0$ (as it corresponds to the variance of a random variable defined as $\gamma_1 W_{X_1} - \gamma_2 \, W_{X_2}$).

Therefore, it follows that, $E[MSE_r] \geq E[MSE_c]$.

\section{Synthetic data experiment details}

\subsection{Regression task illustrations details}

Here, we present the details of the regression task experiments presented in the main text. We ran two experiments: the first based on correctly specified models, and the second based on mispecified models. In both experiments, we simulated correlated error terms, $\bfmW_A$ and $\bfmW_X$, from bivariate normal distributions,
\begin{align}
\bfmW_A^o &\sim \mbox{N}_2\left(
\begin{pmatrix}
0 \\
0 \\
\end{pmatrix}\, , \,
\begin{pmatrix}
1 & \rho_A \\
\rho_A & 1 \\
\end{pmatrix} \right)~, \\
\bfmW_X^o &\sim \mbox{N}_2\left(
\begin{pmatrix}
0 \\
0 \\
\end{pmatrix}\, , \,
\begin{pmatrix}
1 & \rho_X \\
\rho_X & 1 \\
\end{pmatrix} \right)~.
\label{eq:correlated.errors}
\end{align}
In the first experiment, the confounders, output and input variables were generated according to,
\begin{align}
A_j^o &= \mu_{A_j} + W_{A_j}^o~,  \label{eq:correct.model.1} \\
Y^o &= \mu_Y + \beta_{YA_1} \, {A_1^o} + \beta_{YA_2} \, {A_2^o} + W_Y^o~,  \label{eq:correct.model.2} \\
X_j^o &= \mu_{X_j} + \beta_{{X_j}{A_1}} \, {A_1^o} + \beta_{{X_j}{A_2}} \, {A_2^o} + \beta_{{X_j}{Y}} \, {Y^o} + W_{X_j}^o~, \label{eq:correct.model.3}
\end{align}
while in the second (i.e., the mispecified case) they were generated according to,
\begin{align}
A_j^o &= \mu_{A_j} + W_{A_j}^o~, \label{eq:mispecified.model.1} \\
Y^o &= \mu_Y + \beta_{YA_1} \, {A_1^o}^2 + \beta_{YA_2} \, {A_2^o}^2 + W_Y^o~, \label{eq:mispecified.model.2} \\
X_j^o &= \mu_{X_j} + \beta_{{X_j}{A_1}} \, {A_1^o}^2 + \beta_{{X_j}{A_2}} \, {A_2^o}^2 + \beta_{{X_j}{Y}} \, {Y^o}^2 + W_{X_j}^o~,\label{eq:mispecified.model.3}
\end{align}
where $j = 1, 2$ and $W_Y^o \sim \mbox{N}(0 \, , \, \sigma^2_Y)$.

For each experiment, we performed 1000 simulations as follows:
\begin{enumerate}
\item Randomly sampled the simulation parameters from uniform distributions, with the intercept parameters $\mu_{A_1}$, $\mu_{A_2}$, $\mu_{Y}$, and $\mu_{X_j}$ drawn from a $\mbox{U}(-3, 3)$ distribution; the regression coefficients, $\beta_{{Y}{A_1}}$, $\beta_{{Y}{A_2}}$, $\beta_{{X_j}{Y}}$, $\beta_{{X_j}{A_1}}$, and $\beta_{{X_j}{A_2}}$, and the error variance, $\sigma_{Y}^2$, drawn from a $\mbox{U}(1, 3)$ distribution; and the correlations $\rho_A$ and $\rho_X$ from a $\mbox{U}(-0.8, 0.8)$ distribution.
\item Simulated the original data $\bfmA^o$, $Y^o$, and $\bfmX^o$ using the simulation parameters sampled in step 1, according to the models in equations (\ref{eq:correct.model.1})-(\ref{eq:correct.model.3}) in the first experiment, and equations (\ref{eq:mispecified.model.1})-(\ref{eq:mispecified.model.3}) in the second, and then standardized the data to obtain $\bfmA$, $Y$, and $\bfmX$. (Each simulated dataset was composed of 10,000 training and 10,000 test examples.)
\item For each simulated feature, $X_j$, we generated the respective residualized and causality-aware features as described in Sections 2.4 and 2.5 in the main text.
\item For each residual and causality-aware input we computed $Cov(\hat{X}_{r,j}, Y)$ and $Cov(\hat{X}_{c,j}, Y)$.
\item Finally, we trained linear regression models using the residualized and the causality-aware features, and computed the respective test set mean squared errors, $\mbox{MSE}_r$ and $\mbox{MSE}_c$.
\end{enumerate}

\subsection{Classification task illustrations details}

Here, we present the details of classification task experiments presented in the main text. As before, we ran two experiments: the first based on correctly specified models, and the second based on mispecified models. In both experiments, we simulated correlated error terms, $\bfmW_X$, from the bivariate normal distribution,
\begin{equation}
\bfmW_X \sim \mbox{N}_2\left(
\begin{pmatrix}
0 \\
0 \\
\end{pmatrix}\, , \,
\begin{pmatrix}
1 & \rho_X \\
\rho_X & 1 \\
\end{pmatrix} \right)~.
\label{eq:correlated.errors}
\end{equation}
and simulate correlated binary confounder variables from a bivariate Bernoulli distribution~\cite{dai2013}, with probability density function given by,
\begin{equation}
p(A_1^o, A_2^o) \, = \, p_{11}^{a_1 \, a_2} \, p_{10}^{a_1 \, (1 - a_2)} \,  p_{01}^{(1 - a_1) \, a_2} \, p_{00}^{(1 - a_1) \, (1 - a_2)}~,
\end{equation}
where $p_{ij} = P(A_1^o = i, A_2^o = j)$ and $p_{11} + p_{10} + p_{01} + p_{00} = 1$, and the covariance between $A_1^o$ and $A_2^o$ is given by~\cite{dai2013},
\begin{equation}
Cov(A_1^o, A_2^o) \, = \, p_{11} \, p_{00} \, - \, p_{01} \, p_{10}~.
\end{equation}
The binary output data $Y^o$ was generated according to a logistic regression model where,
\begin{align}
P&(Y^o = 1 \mid A_1^o = a_1, A_2^o = a_2) = 1/(1 + \exp{\{ -(\mu_Y + \beta_{YA_1} \, a_1 +  \beta_{YA_2} \, a_2)\}})~.
\end{align}
For the correctly specified experiments, the features $X_j^0$, $j = 1, 2$, where generated according to,
\begin{equation}
X_j^o = \mu_{X_j} + \beta_{{X_j}{A_1}} \, {A_1^o} + \beta_{{X_j}{A_2}} \, {A_2^o} + \beta_{{X_j}{Y}} \, {Y^o} + W_{X_j}^o~. \label{eq:correct.model.class.3}
\end{equation}
For the incorrectly specified experiments, on the other hand, the features were generated as,
\begin{align}
X_j^o = \mu_{X_j} &+ \beta_{{X_j}{Y}{A_1}} \, {Y^o} \, A_1^o + \beta_{{X_j}{Y}{A_2}} \, {Y^o} \, A_2^o + W_{X_j}^o~, \label{eq:incorrect.model.class.3}
\end{align}
containing only interaction terms between $A_k^o$ and $Y^o$.

Each experiment was based on 1000 replications with simulation parameters $\mu_{Y}$ and $\mu_{X_j}$ drawn from a $\mbox{U}(-3, 3)$ distribution; $\beta_{{Y}{A_1}}$, $\beta_{{Y}{A_2}}$, $\beta_{{X_j}{Y}}$, $\beta_{{X_j}{A_1}}$, $\beta_{{X_j}{A_2}}$, $\beta_{{X_j}{Y}{A_1}}$, and $\beta_{{X_j}{Y}{A_2}}$ drawn from a $\mbox{U}(1, 3)$ distribution; $\rho_X \sim \mbox{U}(-0.8, 0.8)$; and $p_{11}$, $p_{10}$, $p_{01}$, and $p_{00}$ sampled by randomly splitting the interval $(0, 1)$ into 4 pieces. For each simulated feature, we generated the respective residualized and causality-aware features and trained logistic regression classifiers using the processed features, and computed the respective test set classification accuracies, $\mbox{ACC}_r$ and $\mbox{ACC}_c$.

\section{Evaluating the effectiveness of the confounding adjustment}

Since any inferences draw from the deterministic counterfactual approach employed in the causality-aware adjustment rely on modeling choices, it is essential to evaluate if the proposed adjustment approach is working as expected. Following the approach proposed by~\cite{chaibubneto2019}, we describe how to use conditional independence patterns to evaluate if predictions generated with the causality-aware approach are really free from the observed confounders influence. The key idea is to represent the data generation process of the observed data together with the data generation process giving rise to the predictions as a causal diagram, and compare the conditional independence relations predicted by d-separation against the conditional independence relations observed in the data. Throughout this section we let $\hat{Y}^{ts}$ represent either the predicted outputs in the regression tasks or the predicted probability of the positive class in the classification tasks. Before we present the approach, we first provide additional background that will be needed for these analyses.

\subsection{Additional background}

In a DAG, a \textit{path} corresponds to any unbroken, nonintersecting sequence of edges in the DAG, which may go along or against the direction of the arrows. A path is \textit{d-separated} or \textit{blocked}~\cite{pearl2009} by a set of nodes $\bfmZ$ if and only if: (i) the path contains a chain $V_j \rightarrow V_m \rightarrow V_k$ or a fork $V_j \leftarrow V_m \rightarrow V_k$ such that the middle node $V_m$ is in $\bfmZ$; or (ii) the path contains a collider $V_j \rightarrow V_m \leftarrow V_k$ such that $V_m$ is not in $\bfmZ$ and no descendant of $V_m$ is in $\bfmZ$. Otherwise, the path is \textit{d-connected} or \textit{open}. The joint distribution over a set of random variables is \textit{faithful}~\cite{spirtes2000,pearl2009} to a causal diagram if no conditional independence relations, other than the ones implied by the d-separation criterion are present. The notation $V_1 \nci V_2$ and $V_1 \ci V_2$ represents marginal statistical dependence and independence, respectively. Conditional dependencies and independencies of $V_1$ and $V_2$ given $V_3$ are represented using the notation $V_1 \nci V_2 \mid V_3$ and $V_1 \ci V_2 \mid V_3$, respectively.

\subsection{Conditional independence patterns for an anticausal predictive task}

Figure \ref{fig:ci.anticausal.dag}a presents the causal graph underlying an anticausal prediction task. In this diagram, the black arrows represent the data generation process underlying the observed data, $\bfmX$, $\bfmA$, and $Y$, while the red arrows represent the data generation process giving rise to the test set predictions, $\hat{Y}^{ts}$. Note that $\bfmX^{tr}$, $Y^{tr}$, and $\bfmX^{ts}$, are parents of $\hat{Y}^{ts}$ since the prediction is a function of both the training data used to train a learner, $\{\bfmX^{tr}, Y^{tr}\}$, and of the test set inputs used to generated the predictions. Figure \ref{fig:ci.anticausal.dag}b shows the simplified graph (omitting the $\bfmX^{ts}$ node) for the test set data, where the full paths $\bfmA^{ts} \rightarrow \bfmX^{ts} \rightarrow \hat{Y}^{ts}$ and $\bfmA^{ts} \rightarrow Y^{ts} \rightarrow \bfmX^{ts} \rightarrow \hat{Y}^{ts}$ have been replaced by the simplified paths $\bfmA^{ts} \rightarrow \hat{Y}^{ts}$ and $\bfmA^{ts} \rightarrow Y^{ts} \rightarrow \hat{Y}^{ts}$. Note that in this diagram, $\bfmA^{ts}$ represents a confounder of the prediction $\hat{Y}^{ts}$, since there is an open path from $\bfmA^{ts}$ to $\hat{Y}^{ts}$ that does not go through $Y^{ts}$ (namely, $\bfmA^{ts} \rightarrow \hat{Y}^{ts}$), as well as, an open path from $\bfmA^{ts}$ to $Y^{ts}$ that does not go through $\hat{Y}^{ts}$ (namely, $\bfmA^{ts} \rightarrow Y^{ts} \rightarrow \hat{Y}^{ts}$).
\begin{figure}[!h]
$$
\xymatrix@-1.5pc{
& *+[F-:<10pt>]{\bfmA^{tr}} \ar[dl] \ar[dr] & & (a) & & *+[F-:<10pt>]{\bfmA^{ts}} \ar[dl] \ar[dr] & & & (b) & *+[F-:<10pt>]{\bfmA^{ts}} \ar[dl] \ar[dr] & \\
*+[F-:<10pt>]{\bfmX^{tr}} \ar@[red][drrr] & & *+[F-:<10pt>]{Y^{tr}} \ar[ll] \ar@[red][dr] & & *+[F-:<10pt>]{\bfmX^{ts}} \ar@[red][dl] && *+[F-:<10pt>]{Y^{ts}} \ar[ll] && *+[F-:<10pt>]{\hat{Y}^{ts}} & & *+[F-:<10pt>]{Y^{ts}} \ar[ll]  \\
&&& *+[F-:<10pt>]{\hat{Y}^{ts}} &&& \\
(c) & *+[F-:<10pt>]{\bfmA^{tr}} \ar[dr] & & & & *+[F-:<10pt>]{\bfmA^{ts}} \ar[dr] & & & (d) & *+[F-:<10pt>]{\bfmA^{ts}} \ar[dr] & \\
*+[F-:<10pt>]{\bfmX^{tr}_c} \ar@[red][drrr] & & *+[F-:<10pt>]{Y^{tr}} \ar[ll] \ar@[red][dr] & & *+[F-:<10pt>]{\bfmX^{ts}_c} \ar@[red][dl] && *+[F-:<10pt>]{Y^{ts}} \ar[ll] && *+[F-:<10pt>]{\hat{Y}_c^{ts}} & & *+[F-:<10pt>]{Y^{ts}} \ar[ll]  \\
&&& *+[F-:<10pt>]{\hat{Y}_c^{ts}} &&& \\
}
$$
  \caption{Panel a shows the full causal diagram underlying an anticausal prediction task. Panel b shows the simplified graph omitting the $\bfmX^{ts}$ node. Panels c and d show the full and simplified diagrams for the prediction task based on the causality-aware approach.}
  \label{fig:ci.anticausal.dag}
\end{figure}
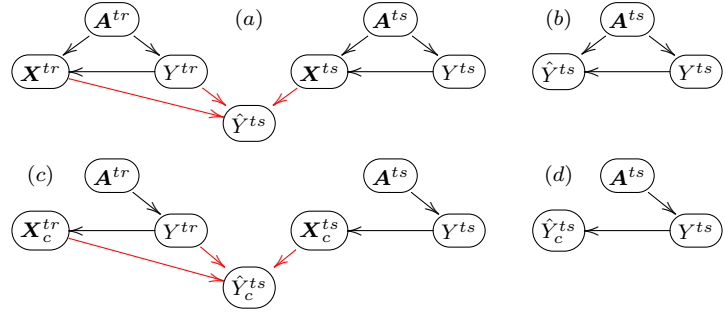

Figure \ref{fig:ci.anticausal.dag}c, on the other hand, shows the causal graph for a learner trained and evaluated with the causality-aware inputs, $\bfmX_c$, while panel d shows the respective simplified graph for the test set. (Note that in panel d $\bfmA$ is no longer a confounder of the predictions $\hat{Y}_c^{ts}$, since the only path connecting $\bfmA^{ts}$ to $\hat{Y}_c^{ts}$ goes through $Y^{ts}$.)

Note that, if the causality-aware adjustment is effective (and faithfulness holds) then, from the application of d-separation~\cite{pearl2009} to the simplified causal graph in Figure \ref{fig:ci.anticausal.dag}d, we would expect to see the following pattern of marginal and conditional (in)dependence relations in the data:
\begin{align*}
&{\hat{Y}_c^{ts}} \nci Y^{ts}~, \\
&{\hat{Y}_c^{ts}} \nci \bfmA^{ts}~, \\
&\bfmA^{ts} \nci Y^{ts}~, \\
&{\hat{Y}_c^{ts}} \nci Y^{ts} \mid \bfmA^{ts}~, \\
&{\hat{Y}_c^{ts}} \ci \bfmA^{ts} \mid Y^{ts}~, \\
&\bfmA^{ts} \nci Y^{ts} \mid \hat{Y}_c^{ts}~,
\end{align*}
where the only conditional independence is given by ${\hat{Y}_c^{ts}} \ci \bfmA^{ts} \mid Y^{ts}$ (where ${\hat{Y}_c^{ts}}$ and $\bfmA^{ts}$ are conditionally independent given $Y^{ts}$, since conditioning on $Y^{ts}$ blocks the path $\bfmA^{ts} \rightarrow Y^{ts} \rightarrow \hat{Y}_c^{ts}$ in Figure \ref{fig:ci.anticausal.dag}d). On the other hand, if the adjustment has failed, we would still expect to see the conditional association ${\hat{Y}_c^{ts}} \nci \bfmA^{ts} \mid Y^{ts}$ in the data.

Figure \ref{fig:colored.MNIST.plus.synthetic} illustrates the application of these sanity checks for the regression task synthetic data experiments for both the correctly specified and mispecified cases with respect to the $A_2$ confounder (the results for the $A_1$ confounder were presented in Figure 5 in the main text). Panel a reports the results for the correctly specified experiment. Note how the distribution of the $\hat{cor}(\hat{Y}_c^{ts}, A^{ts} \mid Y^{ts})$ values were tightly centered around 0, while the distributions for the other marginal and partial correlations were centered above 0. This illustrates that the conditional (in)dependence patterns were consistent with the model in Figure \ref{fig:ci.anticausal.dag}d suggesting that the causality-aware approach effectively removed the direct influence of the color confounder from the predictions $\hat{Y}^{ts}$. Panel b reports the results for the mispecified model experiments. In this case, the results are no longer consistent with Figure \ref{fig:ci.anticausal.dag}d (note the very large spread of the distribution of $\hat{cor}(\hat{Y}_c^{ts}, A^{ts} \mid Y^{ts})$), but rather are consistent with the confounded prediction task in Figure \ref{fig:ci.anticausal.dag}b. These results clearly show that the mispecified regression models failed to remove the direct influence of $A_2$ from the predictions $\hat{Y}_c^{ts}$ in a fair amount of the simulated datasets, and point to the need for more flexible models. In the next section we present and extension based on additive-models.

\begin{figure}[!h]
\centerline{\includegraphics[width=4in]{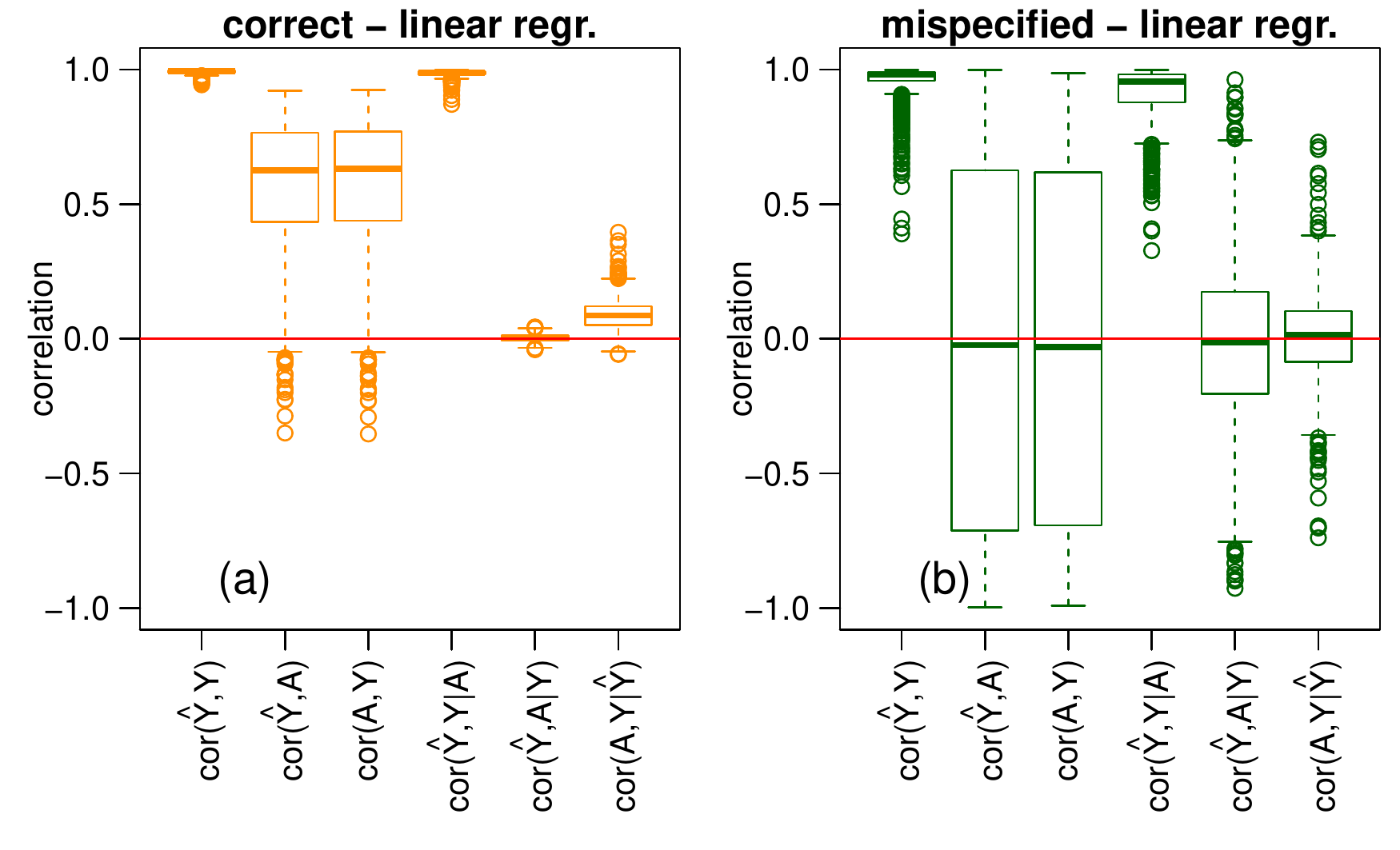}}
\vskip -0.1 in
\caption{Conditional (in)dependence checks for the causality-aware approach.}
\label{fig:colored.MNIST.plus.synthetic}
\end{figure}

\section{Extension to additive-models}

In order to add greater flexibility to our modeling approach (and avoid the often fairly restrictive linearity assumption) we replace the linear models by the more flexible additive-models~\cite{hastie1990}, which are able to capture non-linear relationships between the variables. We denote these extensions as ``additive-model residualization" and ``additive-model causality-aware" approaches.

For the additive-model residualization approach we model each feature $X_j$ using the additive-model,
\begin{equation}
X_j = \mu_{X_j} + \sum_{i=1}^k f_{{X_j}{A_i}}(A_i) + U_{X_j}~,
\end{equation}
and compute the residualized features as,
\begin{equation}
\hat{X}_{r,j} = X_j - \hat{\mu}_{X_j} - \sum_{i=1}^k \hat{f}_{{X_j}{A_i}}(A_i)~,
\end{equation}
where $f_{{V_1}{V_2}}$ represents a scatterplot smoother capable of capturing non-linear relations between variables $V_1$ and $V_2$, and $\hat{f}_{{V_1}{V_2}}$ represents the respective estimated smoother.

For the additive-model causality-aware approach, we fit the following additive model to the training data,
\begin{equation}
X_j^{tr} = \mu_{X_j}^{tr} + f^{tr}_{{X_j}Y}(Y^{tr}) + \sum_{i=1}^k f^{tr}_{{X_j}{A_i}}(A_i^{tr}) + U_{X_j}^{tr}~,
\end{equation}
and compute the training set causality-aware features as,
\begin{equation}
\hat{X}_{c,j}^{tr} = \hat{\mu}_{X_j}^{tr} + \hat{f}^{tr}_{{X_j}Y}(Y^{tr}) + \hat{U}_{X_j}^{tr}~,
\end{equation}
where,
\begin{equation}
\hat{U}_{X_j}^{tr} = X_j^{tr} - \hat{\mu}_{X_j}^{tr} - \hat{f}^{tr}_{{X_j}Y}(Y^{tr}) - \sum_{i=1}^k \hat{f}^{tr}_{{X_j}{A_i}}(A_i^{tr})~,
\end{equation}
while the causality-aware test set features are computed as,
\begin{equation}
\hat{X}_{c,j}^{ts} = X_j^{ts} - \sum_{i=1}^k \hat{f}^{tr}_{{X_j}{A_i}}(A_i^{ts})~,
\end{equation}
where $\hat{f}^{tr}_{{X_j}{A_i}}(A_i^{ts})$ represents the evaluation of the test set confounder data, $A_i^{ts}$, using the respective scatterplot smoother estimated in the training set.

Figure \ref{fig:ci.tests.lin.regr.vs.add.mod} reports a comparison of the conditional (in)dependence patterns of the linear regression causality-aware approach (panels a and b) against the additive-model causality-aware approach (panels c and d) for the synthetic data experiments.
\begin{figure}[!h]
\centerline{\includegraphics[width=\linewidth]{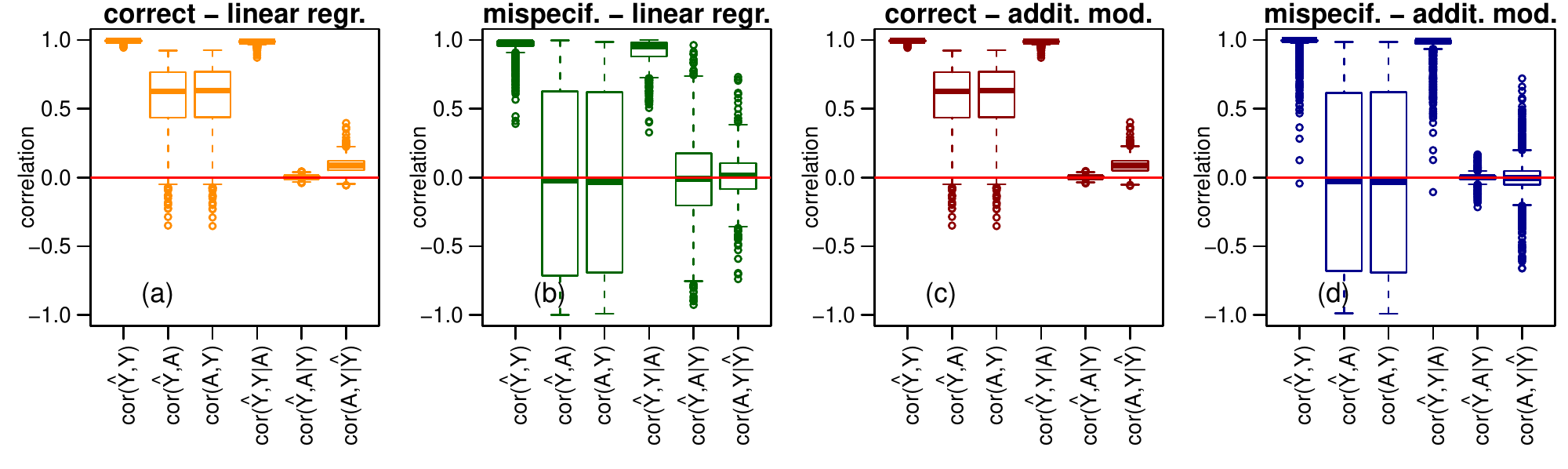}}
\caption{Comparison of conditional (in)dependence checks for the linear regression causality-aware approach vs the additive-model causality-aware approach. Results based on the same 1000 simulated datasets generated as described in the main text, with training and test sets of size 10,000.}
\label{fig:ci.tests.lin.regr.vs.add.mod}
\end{figure}

Note that for the experiments based on correctly specified model, the results based on the additive-model adjustment were quite similar to the results based on the linear regression adjustment (compare panel a vs panel c). This is expected, since the flexible additive-models are able to adapt to the data, so that when the data truly follows a linear model, the additive model will ``mimic" a linear model fit. For the experiments based on mispecified models, on the other hand, we see that the additive-model based causality-aware adjustment was much more effective in removing confounding than the linear model based adjustment (note how $\hat{cor}(\hat{Y}_c^{ts}, A^{ts} \mid Y^{ts})$ distribution is much more tightly centered around 0 in panel d than in panel b).

Figure \ref{fig:scatterplots.lin.regr.vs.add.mod} compares the covariances and mean squared errors from learners trained with the additive-model residualization inputs versus learners trained with the additive-model causality-aware inputs. This empirical comparison shows that the causality-aware approach still outperforms the residualization adjustment when the linear-regression models are replaced by the more flexible additive-models.
\begin{figure}[!h]
\centerline{\includegraphics[width=4in]{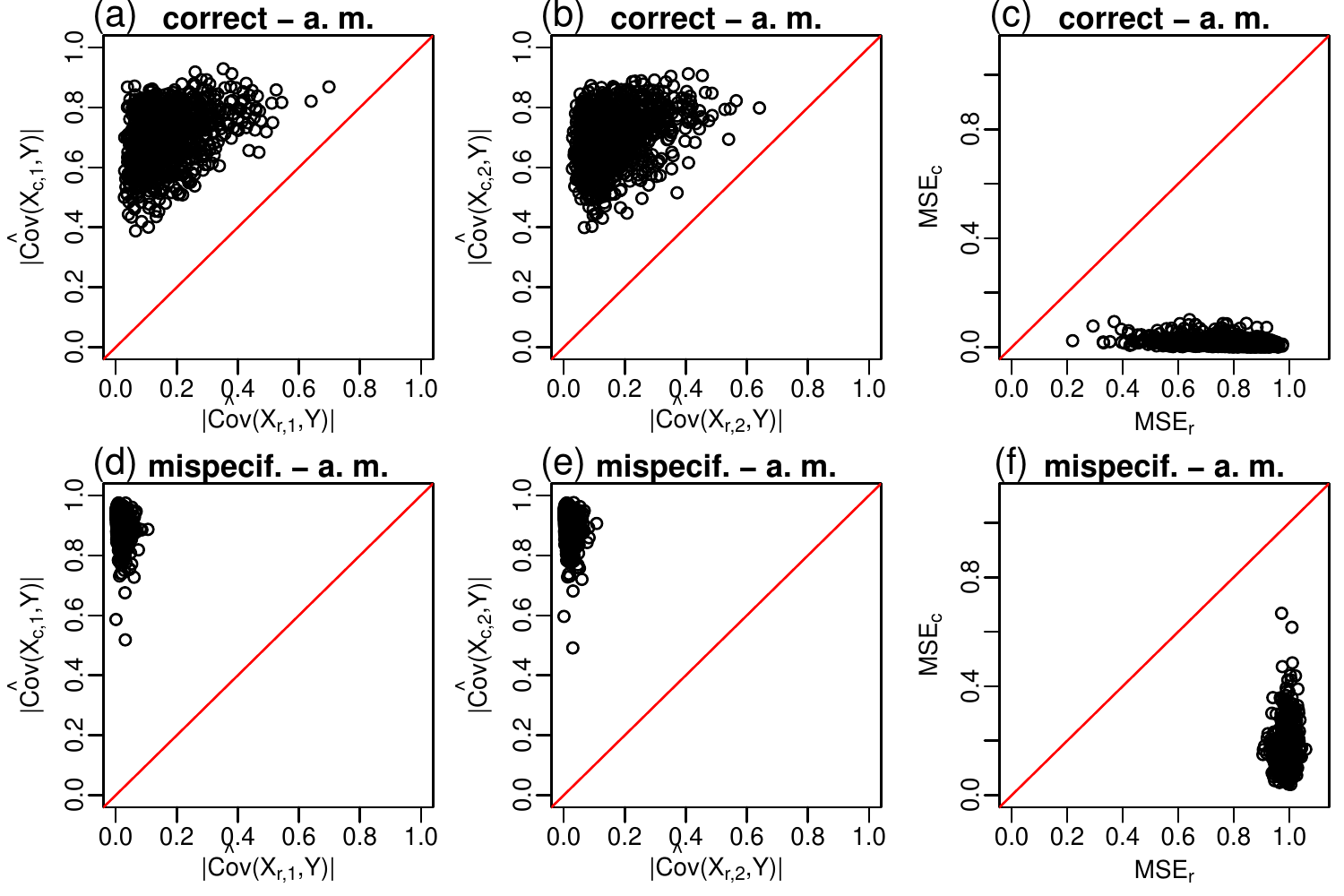}}
\caption{Comparison of the regression task experiments based on additive-model residualization and additive-model causality-aware adjustments. Results based on the same 1000 simulated datasets generated as described in the main text, with training and test sets of size 10,000.}
\label{fig:scatterplots.lin.regr.vs.add.mod}
\end{figure}

\noindent \textbf{Remarks.} While the causality-aware approach based on additive-models was able to effectively combat confounding in these synthetic data experiments, additive models still make the important assumption that the data generation process is additive (what was true in our synthetic data experiments, but which might still be violated in real data settings). In any case, a modeler can always apply the conditional (in)dependence pattern evaluations described above to check if the adjustment is really working or not, and then decide if even more flexible models are still needed.

Finally, we would like to point out that we have focused the application of these conditional (in)dependence evaluations only on data from the causality-aware approach, but not on the residualization approach, because the residualization approach generates data that is unfaithful to any causal diagram describing an anticausal prediction task (that is, the conditional (in)dependence patterns generated by the residualization approach are not consistent with any causal diagrams where $Y^{ts}$ has a causal influence on $\hat{Y}^{ts}$). In the next subsection, we describe this point in more detail.

\subsection{A note on the unfaifulness of the residualization approach in anticausal prediction tasks}

Here, we describe why the residualization approach generates data that is unfaithful to anticausal prediction tasks using a toy example. Consider again the anticausal prediction task,
$$
\xymatrix@-1.0pc{
& *+[F-:<10pt>]{A} \ar[dl] \ar[dr] &  \\
*+[F-:<10pt>]{X} & & *+[F-:<10pt>]{Y} \ar[ll] \\
}
$$
where, for simplicity, we consider a single confounder $A$ and a single feature $X$, and we assume that the true data generation process is given by the standardized linear structural models
\begin{align}
A &= U_A~, \label{eq:str.eq.C} \\
Y &= \theta_{YA} \, A + U_Y~, \label{eq:str.eq.X} \\
X &= \theta_{XA} \, A + \theta_{XY} \, Y + U_X~, \label{eq:str.eq.Y}
\end{align}
where all variables have mean 0 and variance 1. Assuming faithfulness, we have that all marginal and partial covariances are non-zero and given by,
\begin{align*}
Cov(X, Y) &= \theta_{XY} + \theta_{XA} \, \theta_{AY}~, \\
Cov(X, A) &= \theta_{XA} + \theta_{YA} \, \theta_{XY}~, \\
Cov(A, Y) &= \theta_{YA}~, \\
Cov(X, Y \mid A) &= Cov(X, Y) - Cov(X, A) \, Cov(A, Y)~, \\
Cov(X, A \mid Y) &= Cov(X, A) - Cov(X, Y) \, Cov(A, Y)~, \\
Cov(A, Y \mid X) &= Cov(A, Y) - Cov(X, A) \, Cov(X, Y)~.
\end{align*}

Now, consider the residualized input, $\hat{X}_r = X - \hat{\omega}_{XA} A$, and suppose that sample size is large, so that it converges to $X_r = X - \omega_{XA} A$, where $\omega_{XA} = \theta_{XA} + \theta_{YA} \theta_{XY}$. Hence, by construction, we have that $Cov(X_r, A) = 0$ since,
\begin{align*}
Cov(X_r, A) &= Cov(X - \omega_{XA} A, A) = Cov(X, A) - \omega_{XA} Cov(A, A) \\
&= Cov(X, A) - \omega_{XA} = \theta_{XA} + \theta_{YA} \, \theta_{XY} - \omega_{XA} = 0~.
\end{align*}
Hence, we see that,
\begin{align*}
Cov(X_r, Y) &= Cov(X - \omega_{XA} A, Y) = Cov(X, Y) - \omega_{XA} Cov(A, Y)~, \\
&= \theta_{XY} (1 - \theta_{YA}^2) \\
Cov(X_r, A) &= 0, \\
Cov(A, Y) &= \theta_{YA}~, \\
Cov(X_r, Y \mid A) &= Cov(X_r, Y)~, \\
Cov(X_r, A \mid Y) &= - Cov(X_r, Y) \, Cov(A, Y)~, \\
Cov(A, Y \mid X_r) &= Cov(A, Y)~.
\end{align*}

Quite importantly, observe that the above marginal and partial covariances show that the conditional (in)dependence pattern generated by the residualization approach is given by,
\begin{equation*}
X_r \nci Y~, \;\;\;
X_r \ci A~, \;\;\;
A \nci Y~, \;\;\;
X_r \nci Y \mid A~, \;\;\;
X_r \nci A \mid Y~, \;\;\;
A \nci Y \mid X_r~,
\end{equation*}
which is consistent with the causal model,
$$
\xymatrix@-1.0pc{
& *+[F-:<10pt>]{A} \ar[dr] &  \\
*+[F-:<10pt>]{X_r} \ar[rr] & & *+[F-:<10pt>]{Y} \\
}
$$
since $X_r$ and $A$ are marginally independent, but become conditionally associated when we condition on $Y$. Consequently, when we train a learner with the residualized features, the conditional independence relations among the $\hat{Y}_r^{ts}$, $A^{ts}$, and $Y^{ts}$ will be consistent with the model,
$$
\xymatrix@-1.5pc{
& *+[F-:<10pt>]{A^{tr}} \ar[dr] & & & & *+[F-:<10pt>]{A^{ts}} \ar[dr] & & &  \\
*+[F-:<10pt>]{X_r^{tr}} \ar@[red][drrr] \ar[rr] & & *+[F-:<10pt>]{Y^{tr}} \ar@[red][dr] & & *+[F-:<10pt>]{X_r^{ts}} \ar@[red][dl] \ar[rr] && *+[F-:<10pt>]{Y^{ts}} \\
&&& *+[F-:<10pt>]{\hat{Y}_r^{ts}} && \\
}
$$
so that we should expect to see the following conditional independence relations,
\begin{align*}
&{\hat{Y}_r^{ts}} \nci Y^{ts}~, \\
&{\hat{Y}_r^{ts}} \ci A^{ts}~, \\
&A^{ts} \nci Y^{ts}~, \\
&{\hat{Y}_r^{ts}} \nci Y^{ts} \mid A^{ts}~, \\
&{\hat{Y}_r^{ts}} \nci A^{ts} \mid Y^{ts}~, \\
&A^{ts} \nci Y^{ts} \mid \hat{Y}_r^{ts}~.
\end{align*}

Figure \ref{fig:ci.patterns.resid} reports the observed conditional independence patterns for the residualized features generated from the synthetic data experiments based on the correctly specified models, and illustrate this point.
\begin{figure}[!h]
\centerline{\includegraphics[width=3.5in]{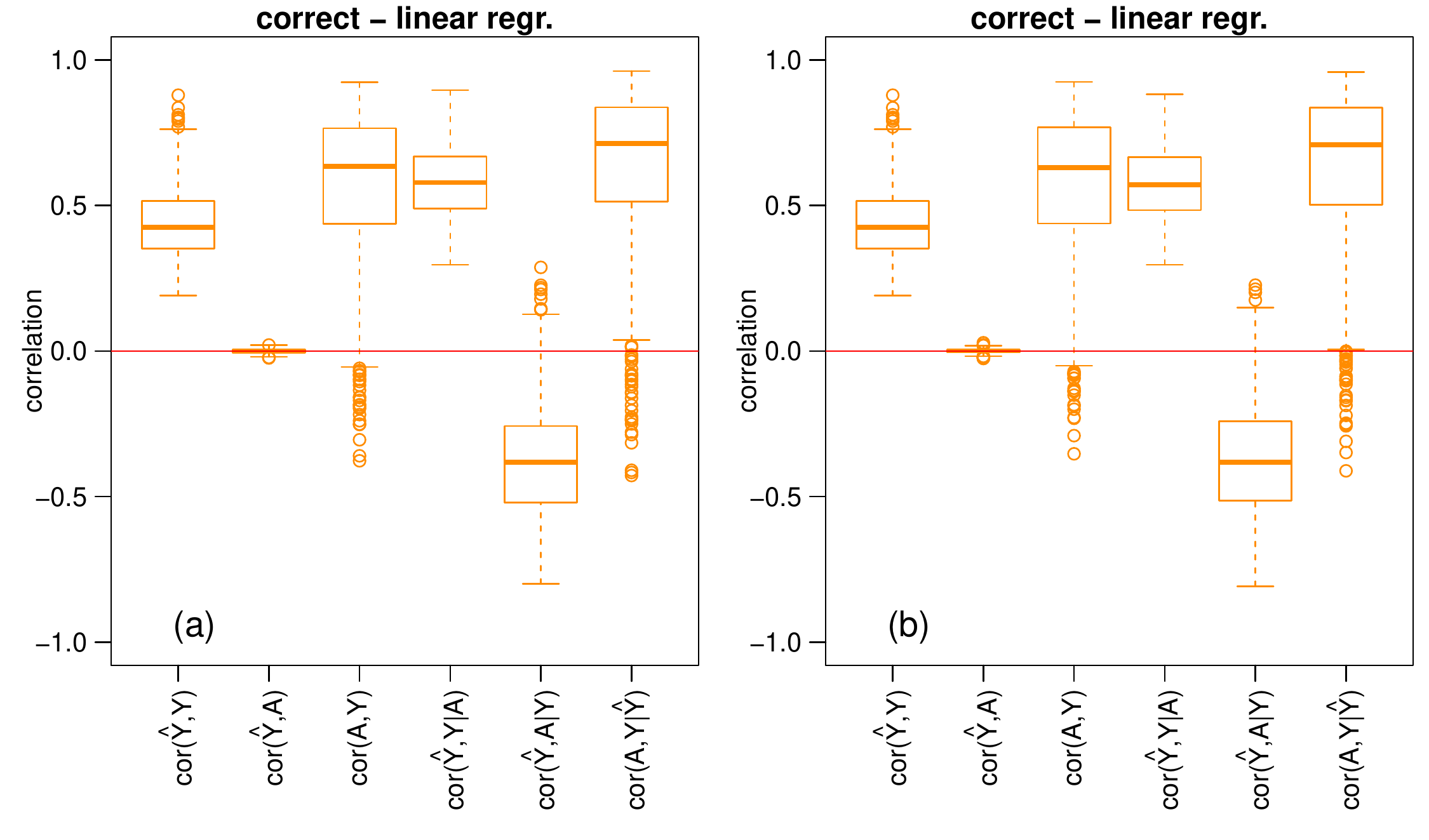}}
\caption{Conditional independence patterns for the residualization approach. Panels a and b show the results for the $A_1$ and $A_2$ confounders, respectively.}
\label{fig:ci.patterns.resid}
\end{figure}

These observations clearly show that it doesn't make sense to apply the conditional (in)dependenve pattern evaluation approach to data processed with the residualization approach, since the conditional (in)dependenve pattern is not faithful to the true anticausal model, where the output has a causal influence on the predictions. Note, however, that the fact that the residualization approach generates unfaithful data is not really surprising, given that residualization is not a causality-inspired approach.

\section{Stability comparisons}

\subsection{Expected MSE values for the toy model used in the stability experiments}

Here, we present the expected MSE values for the causality-aware and residualization approches, for the toy model used in the dataset shift experiments. These analyses show that, contrary to the residualization approach, the expected MSE for the causality-aware adjustment does not depend on the covariance between $A$ and $Y$ ($\sigma_{AY}$). These results explain the better stability of the causality-aware approach w.r.t. dataset shifts in the $P(A, Y)$ distribution.

In our illustrations we consider the model,
$$
\xymatrix@-1.2pc{
& *+[F-:<10pt>]{A} \ar[dl]_{\beta_{XA}} \ar@/^1.0pc/@{<->}[dr]^{\sigma_{AY}} &  \\
*+[F-:<10pt>]{X} && *+[F-:<10pt>]{Y} \ar[ll]^{\beta_{XY}} \\
}
$$
where we assume that $A$ and $Y$ are correlated random variables with expectation and covariance given by,
\begin{equation}
E
\begin{pmatrix}
A \\
Y \\
\end{pmatrix}
=
\begin{pmatrix}
0 \\
0 \\
\end{pmatrix}~, \hspace{0.5cm}
\bfSigma_{A,Y} =
\begin{pmatrix}
Var(A) & Cov(A, Y) \\
Cov(A, Y) & Var(Y) \\
\end{pmatrix}
=
\begin{pmatrix}
\sigma_{AA} & \sigma_{AY} \\
\sigma_{AY} & \sigma_{YY} \\
\end{pmatrix}~,
\end{equation}
and $X$ follows the regression model,
\begin{equation}
X = \beta_{XY} \, Y + \beta_{XA} \, A + U_X~, \hspace{0.3cm} E(U_X) = 0~, \hspace{0.3cm} Var(U_X) = \sigma_{X}^2~,
\end{equation}
Observe that $E(X)$ equals 0 for this model. We assume that $\beta_{XY}$ and $\beta_{XA}$ are stable, but $\bfSigma_{A,Y}$ is unstable between training and test sets (i.e., $\bfSigma_{A,Y}^{tr} \not= \bfSigma_{A,Y}^{ts}$).

Now, let $\hat{Y} = X^{ts} \hat{\beta}_{YX}^{tr}$ represent the prediction of linear model, where $X^{ts}$ represents the test set feature, and $\hat{\beta}_{YX}^{tr}$ represents the regression coefficients estimated from the training set. By definition, the expected mean squared error of the prediction is given by,
\begin{align}
E[MSE] &= E[(Y^{ts} - \hat{Y})^2] = E[(Y^{ts})^2] + E[\hat{Y}^2] - 2 E[\hat{Y} Y^{ts}]~, \\
&= E[(Y^{ts})^2] + (\hat{\beta}_{YX}^{tr})^2 E[(X^{ts})^2] - 2 \hat{\beta}_{YX}^{tr} E[X^{ts} Y^{ts}]~, \\
&= Var(Y^{ts}) + (\hat{\beta}_{YX}^{tr})^2 Var(X^{ts}) - 2 \hat{\beta}_{YX}^{tr} Cov(X^{ts}, Y^{ts}),
\end{align}
where the third equality follows from the fact that $E[Y^{ts}] = E[X^{ts}] = 0$. Observe, as well, that because expectation is taken w.r.t. the test set, we have that $\hat{\beta}_{YX}^{tr}$ is a constant.

Starting with the causality-aware approach, and assuming that sample size goes to infinity, so that $\hat{X}_c = X - \hat{\beta}_{XA} A$ converges to $X_c = X - \beta_{XA} A = \beta_{XY} Y + U_X$, we have that,
\begin{align}
&Var(X_{c}^{ts}) = Var(\beta_{XY} \, Y^{ts} + U_X^{ts}) = \beta_{XY}^2 Var(Y^{ts}) + Var(U_X^{ts}) = \sigma_{X}^{2} + \beta_{XY}^2 \, \sigma_{YY}^{ts} \\
&Cov(X_{c}^{ts}, Y^{ts}) = Cov(\beta_{XY} \, Y^{ts} + U_X^{ts}, Y^{ts}) = \beta_{XY} Var(Y^{ts}) = \beta_{XY} \, \sigma_{YY}^{ts}
\end{align}
so that,
\begin{equation}
E[MSE_c] = \sigma_{YY}^{ts} + (\hat{\beta}_{YX}^{c,tr})^2 (\sigma^2_X + \beta_{XY}^2 \, \sigma_{YY}^{ts}) - 2 \hat{\beta}_{YX}^{c,tr} \beta_{XY} \, \sigma_{YY}^{ts}~,
\end{equation}
is not a function of $\sigma_{AY}$ (although it still depends on $\sigma_{YY}^{ts}$).

For the residualization approach, on the other hand, we have that $\hat{X}_r = X - \hat{\omega}_{XA} A$ converges to $X_r = X - \omega_{XA} A$ so that the variance of $X_{r}^{ts}$ is given by,
\begin{align}
Var(X_{r}^{ts}) &= Var(X^{ts} - \omega_{XA} \, A^{ts}) \nonumber \\
&= Var(X^{ts}) + \omega_{XA}^2 \, Var(A_{ts}) - 2 \, \omega_{XA} \, Cov(X^{ts}, A^{ts}) \nonumber \\
&= Var(X^{ts}) + \frac{Cov(X^{ts}, A^{ts})^2}{Var(A^{ts})^2} \, Var(A^{ts}) - 2 \, \frac{Cov(X^{ts}, A^{ts})}{Var(A^{ts})} \, Cov(X^{ts}, A^{ts}) \nonumber \\
&= Var(X^{ts}) - \frac{Cov(X^{ts}, A^{ts})^2}{Var(A^{ts})} \nonumber \\
&= Var(X^{ts}) - \frac{Cov(\beta_{XY} \, Y^{ts} + \beta_{XA} \, A^{ts} + U_X^{ts}, A^{ts})^2}{Var(A^{ts})} \nonumber \\
&= Var(X^{ts}) - \frac{[\beta_{XY} \, Cov(Y^{ts}, A^{ts}) + \beta_{XA} \, Var(A^{ts})]^2}{Var(A^{ts})} \nonumber \\
&= Var(X^{ts}) - \frac{(\beta_{XY} \, \sigma_{AY}^{ts} + \beta_{XA} \, \sigma_{AA}^{ts})^2}{\sigma_{AA}^{ts}} \nonumber \\
&= \sigma_{X}^{2} + \beta_{XY}^2 \, \sigma_{YY}^{ts} - \frac{\beta_{XY}^2 \, (\sigma_{AY}^{ts})^2}{\sigma_{AA}^{ts}}
\end{align}
where the third equality follows from the fact that $\omega_{XA}$ represents the coefficient of the regression of $X$ on $A$, and corresponds to $Cov(X^{ts}, A^{ts})/Var(A^{ts})$, while the last equality follows from the fact that,
\begin{align}
Var(X^{ts}) &= Var(\beta_{XY} \, Y^{ts} + \beta_{XA} \, A^{ts} + U_X^{ts}) \nonumber \\
&= \beta_{XY}^2 \, Var(Y^{ts}) + \beta_{XA}^2 \, Var(A^{ts}) + Var(U_X^{ts}) + 2 \beta_{XY} \, \beta_{XA} \, Cov(Y^{ts}, A^{ts}) \nonumber \\
&= \sigma_{X}^2 + \beta_{XY}^2 \, \sigma_{YY}^{ts} + \beta_{XA}^2 \, \sigma_{AA}^{ts} + 2 \beta_{XY} \, \beta_{XA} \, \sigma_{AY}^{ts}~.
\end{align}
Computation of the covariance shows that,
\begin{align}
Cov(X_{r}^{ts}, Y^{ts}) &= Cov(X^{ts}, Y^{ts}) - \omega_{XA} \, Cov(A^{ts}, Y^{ts}) \nonumber \\
&= Cov(\beta_{XY} \, Y^{ts} + \beta_{XA} \, A^{ts} + U_X^{ts}, Y^{ts}) - \frac{Cov(X^{ts}, A^{ts})}{Var(A^{ts})} \, \sigma_{AY}^{ts} \nonumber \\
&= \beta_{XY} \, Var(Y^{ts}) + \beta_{XA} \, Cov(A^{ts}, Y^{ts}) - \frac{\beta_{XY} \, \sigma_{AY}^{ts} + \beta_{XA} \, \sigma_{AA}^{ts}}{\sigma_{AA}^{ts}} \, \sigma_{AY}^{ts} \nonumber \\
&= \beta_{XY} \, \sigma_{YY}^{ts} + \beta_{XA} \, \sigma_{AY}^{ts} - \frac{\beta_{XY} \, \sigma_{AY}^{ts} + \beta_{XA} \, \sigma_{AA}^{ts}}{\sigma_{AA}^{ts}} \, \sigma_{AY}^{ts} \nonumber \\
&= \beta_{XY} \, \sigma_{YY}^{ts} - \frac{\beta_{XY} \, (\sigma_{AY}^{ts})^2}{\sigma_{AA}^{ts}}
\end{align}
Therefore, we have that,
\begin{align*}
E[MSE_r] &= \sigma_{YY}^{ts} + (\hat{\beta}_{YX}^{r,tr})^2 \left(\sigma^2_X + \beta_{XY}^2 \, \sigma_{YY}^{ts} - \frac{\beta_{XY}^2 \, (\sigma_{AY}^{ts})^2}{\sigma_{AA}^{ts}}\right) \, - \\
&\;\;\;-2 \hat{\beta}_{YX}^{r,tr} \left( \beta_{XY} \, \sigma_{YY}^{ts} - \frac{\beta_{XY} \, (\sigma_{AY}^{ts})^2}{\sigma_{AA}^{ts}}\right)~,
\end{align*}
is a function of $\sigma_{AY}$ (and of $\sigma_{AA}$ and $\sigma_{YY}$, as well).

These results show when $\beta_{XY}$ and $\beta_{XA}$ are stable but $\bfSigma_{A,Y}$ is not, we have that $E[MSE_c]$ is still inherently more stable than $E[MSE_r]$, since the latter will vary with $\sigma_{AY}$, $\sigma_{AA}$ and $\sigma_{YY}$, while $E[MSE_c]$ is stable w.r.t. shifts on $\sigma_{AY}$ on $\sigma_{AA}$.

\subsection{Expected MSE values in the general case}

As described in~\cite{achaibubneto2020a} the expected MSE value for an arbitrary anticausal prediction tasks based on linear models is given by,
\begin{align*}
E[MSE] \, &=  \, Var(Y^{ts}) \, + \sum_{j=1}^{p} (\hat{\beta}_j^{tr})^2 Var(X_{j}^{ts}) + 2 \sum_{j < k} \hat{\beta}_j^{tr} \hat{\beta}_k^{tr} Cov(X_{j}^{ts}, X_{k}^{ts}) \, - \\
 &- 2 \sum_{j=1}^{p} \hat{\beta}_j^{tr} Cov(X_{j}^{ts}, Y^{ts})~.
\end{align*}
where, for the causality-aware approach, we have that the quatities,
\begin{align*}
Var(X_{c,j}^{ts}) &= Var(\beta_{{X_j}Y} \, Y^{ts} + U_{X_j}^{ts}) = \sigma^2_{X_j} + \beta_{{X_j}Y}^2 \, Var(Y^{ts})~,
\end{align*}
\begin{align*}
Cov(X_{c,j}^{ts}, X_{c,k}^{ts}) &= Cov(\beta_{{X_j}Y} \, Y^{ts} + U_{X_j}^{ts}, \beta_{{X_k}Y} \, Y^{ts} + U_{X_k}^{ts}) \\
&= \beta_{{X_j}Y} \, \beta_{{X_k}Y} \, Var(Y^{ts}) + Cov(U_{X_j}^{ts}, U_{X_k}^{ts})~,
\end{align*}
\begin{align*}
Cov(X_{c,j}^{ts}, Y^{ts}) &= Cov(\beta_{{X_j}Y} \, Y^{ts} + U_{X_j}^{ts}, Y^{ts}) = \beta_{{X_j}Y} \, Var(Y^{ts})~,
\end{align*}
do not depend on $Cov(A^{ts}, Y^{ts})$.

Now, because $\bfmX_r$ can be re-expressed as,
\begin{align*}
\bfmX_r &= \bfmX - \bfOmega_{XA} \, \bfmA~, \\
&= \bfGamma_{XA} \, \bfmA  + \bfGamma_{XY} \, Y + \bfmW_X - (\bfGamma_{XA} + \bfGamma_{XY} \, \bfGamma_{YA}) \, \bfmA~, \\
&= \bfGamma_{XY} \, Y + \bfmW_X - \bfGamma_{XY} \, \bfGamma_{YA} \, \bfmA~, \\
&= \bfmX_c - \bfGamma_{XY} \, \bfGamma_{YA} \, \bfmA~,
\end{align*}
we have that,
\begin{align}
Cov(\bfmX_r) &= Cov(\bfmX_c - \bfGamma_{XY} \, \bfGamma_{YA} \, \bfmA \, , \, \bfmX_c - \bfGamma_{XY} \, \bfGamma_{YA} \, \bfmA) \nonumber \\
&= Cov(\bfmX_c) - Cov(\bfmX_c, \bfmA) \, \bfGamma_{YA}^T \, \bfGamma_{XY}^T \, - \nonumber \\
&\;\;\;\;-\bfGamma_{XY} \, \bfGamma_{YA} \, Cov(\bfmA, \bfmX_c) + \bfGamma_{XY} \, \bfGamma_{YA} \, Cov(\bfmA) \, \bfGamma_{YA}^T \, \bfGamma_{XY}^T~ \nonumber \\
&= Cov(\bfmX_c) - \bfGamma_{XY} \, Cov(Y, \bfmA) \, \bfGamma_{YA}^T \, \bfGamma_{XY}^T \, - \nonumber \\
&\;\;\;\;-\bfGamma_{XY} \, \bfGamma_{YA} \, Cov(\bfmA, Y) \, \bfGamma_{XY}^T + \bfGamma_{XY} \, \bfGamma_{YA} \, Cov(\bfmA) \, \bfGamma_{YA}^T \, \bfGamma_{XY}^T~ \nonumber \\
&= Cov(\bfmX_c) - 2 \, \bfGamma_{XY} \, Cov(Y, \bfmA) \, \bfGamma_{YA}^T \, \bfGamma_{XY}^T \, + \bfGamma_{XY} \, \bfGamma_{YA} \, Cov(\bfmA) \, \bfGamma_{YA}^T \, \bfGamma_{XY}^T~, \label{eq:cov.Xr.as.function.of.Xc}
\end{align}
where the third equality follows from the fact that,
\begin{equation*}
Cov(\bfmX_c, \bfmA) = Cov(\bfGamma_{XY} \, Y + \bfmW_X, \bfmA) = \bfGamma_{XY} \, Cov(Y, \bfmA)~,
\end{equation*}
and the fourth equality from the fact that $Cov(Y, \bfmA) \, \bfGamma_{YA}^T = \bfGamma_{YA} \, Cov(\bfmA, Y)$ is a scalar.

From equation (\ref{eq:cov.Xr.as.function.of.Xc}) it is clear that both $Var(X_{r,j}^{ts})$ (which corresponds to the $j$-th diagonal element of $Cov(\bfmX_r)$) and $Cov(X_{r,j}^{ts}, X_{r,k}^{ts})$ (which corresponds to the $j,k$-th (off-diagonal) element of $Cov(\bfmX_r)$) are still functions of $Cov(Y, \bfmA)$. Similarly, note that $Cov(X_{r,j}^{ts}, Y^{ts})$ is also still a function of $Cov(\bfmA, Y)$ since,
\begin{align*}
Cov(\bfmX_r, Y) &= Cov(\bfmX_c - \bfGamma_{XY} \, \bfGamma_{YA} \, \bfmA \, , \, Y)~, \\
&= Cov(\bfmX_c, Y) - \bfGamma_{XY} \, \bfGamma_{YA} \, Cov(\bfmA \, , \, Y)~, \\
&= Cov(\bfGamma_{XY} \, Y + \bfmW_X, Y) - \bfGamma_{XY} \, \bfGamma_{YA} \, Cov(\bfmA \, , \, Y)~, \\
&= \bfGamma_{XY} \, Var(Y) - \bfGamma_{XY} \, \bfGamma_{YA} \, Cov(\bfmA \, , \, Y)~.
\end{align*}

Therefore, we have that for the residualization approach, the expected mean squared error is still a function of $Cov(Y^{ts}, \bfmA^{ts})$ and will, therefore, be unstable w.r.t. shifts in this quantity.

\subsection{Stability experiments}

We performed two stability experiments. In the first we kept $Var(Y^{ts})$ constant across the test sets, while in the second we let $Var(Y^{ts})$ vary across the test sets.

Each experiment was based in 1,000 replications. In our first simulation experiment, for each replication we:
\begin{enumerate}
\item Sampled the causal effects $\beta_{XY}$ and $\beta_{XA}$ from a $U(-3, 3)$ distribution, and the training set covariance $\sigma_{AY}^{tr}$ from a $U(-0.8, 0.8)$ distribution.
\item Generated training data ($n = 10,000$) by first sampling,
\begin{equation}
\begin{pmatrix}
A^{tr} \\
Y^{tr} \\
\end{pmatrix}\,
\sim \mbox{N}_2\left(
\begin{pmatrix}
0 \\
0 \\
\end{pmatrix}\, , \,
\begin{pmatrix}
1 & \sigma_{AY}^{tr} \\
\sigma_{AY}^{tr} & 1 \\
\end{pmatrix} \right)~,
\end{equation}
and then generating $X^{tr} = \beta_{XA} \, A^{tr} + \beta_{XY} \, Y^{tr} + U_X^{tr}$ with $U_X^{tr} \sim N(0, 1)$.
\item Generated 9 distinct test sets, where each test set dataset ($n=10,000$) was generated by first sampling,
\begin{equation}
\begin{pmatrix}
A^{ts} \\
Y^{ts} \\
\end{pmatrix}\,
\sim \mbox{N}_2\left(
\begin{pmatrix}
0 \\
0 \\
\end{pmatrix}\, , \,
\begin{pmatrix}
\sigma_{AA}^{ts} & \sigma_{AY}^{ts} \\
\sigma_{AY}^{ts} & \sigma_{YY}^{ts} \\
\end{pmatrix} \right)~,
\end{equation}
and then generating $X^{ts} = \beta_{XA} \, A^{ts} + \beta_{XY} \, Y^{ts} + U_X^{ts}$ with $U_X^{ts} \sim N(0, 1)$. In order to generate dataset shifts, the covariances between $A^{ts}$ and $Y^{ts}$ and the variances of $A^{ts}$ were set, respectively, to $\sigma_{AY}^{ts} = \{-0.8$, $-0.6$, $-0.2$, 0, 0.2, 0.4, 0.6, $0.8 \}$ and $\sigma_{AA}^{ts} = \{1.00$, 1.25, 1.50, 1.75, 2.00, 2.25, 2.50, 2.75, $3.00\}$ across the 9 distinct test sets, while the variance of $Y_{ts}$ was fixed at $\sigma_{YY} = 1$.
\item Processed the training and the test features as described in Sections 2.4 and 2.5 in the main text to generate the residualized and causality-aware features.
\item Trained linear regression models using the residualized and causality-aware features and evaluated the performance of each of the trained models on each of the 9 test sets.
\end{enumerate}

Our second experiment, was run as described above, except that we let $\sigma_{YY}^{ts}$ vary according to $\{1.00, 1.25, 1.50, 1.75, 2.00, 2.25, 2.50, 2.75, 3.00\}$ across the 9 test sets. Figure \ref{fig:dataset.shift.2} reports the results and shows that, while $MSE_c$ also changed across the test sets, the causality-aware approach was still much more stable than residualization.

\begin{figure}[!h]
\centerline{\includegraphics[width=5in]{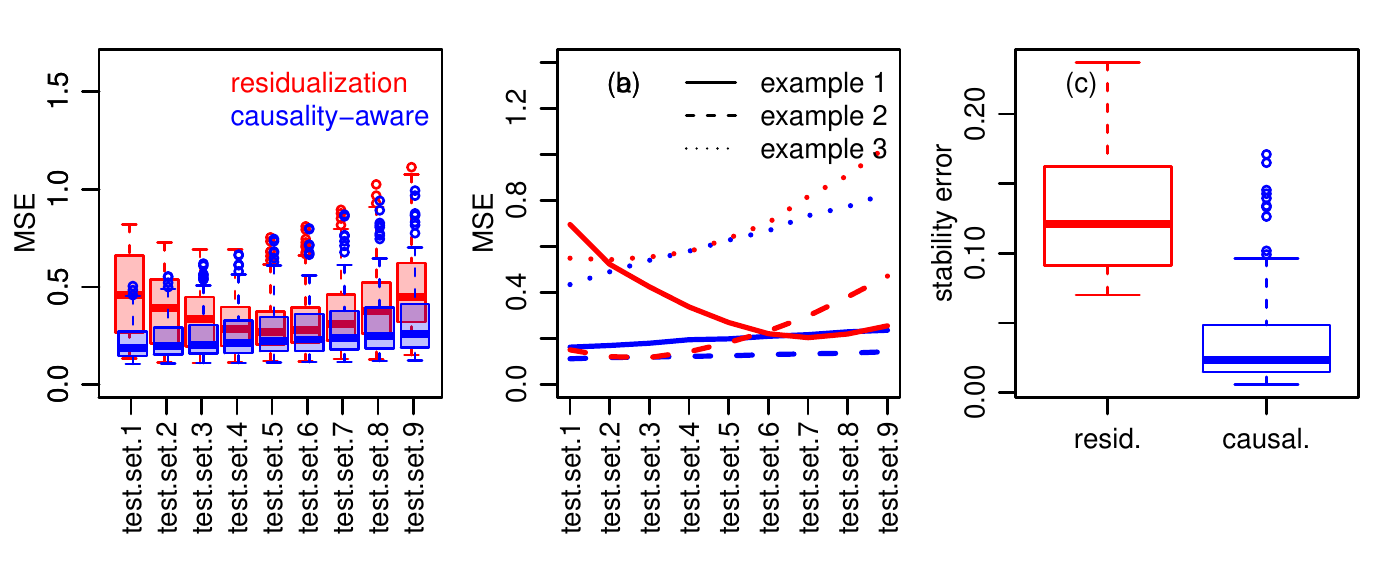}}
\caption{Stability illustrations, with increasing $Var(Y^{ts})$.}
\label{fig:dataset.shift.2}
\end{figure}


\section{A note on anticausal neuroimage disgnostic applications}

This paper focus on anticausal prediction tasks. Diagnostic applications based on neuroimaging data, represent a clear example of an anticausal prediction task. In these applications, the outcome variable represents the disease status (e.g., healthy cognition versus mild cognitive impairment versus full blown Alzheimer's disease) while the inputs represent features extracted from the neuroimages. Because individuals at different stages of the disease trajectory show structural brain differences, we have that features extracted from neuroimages can be used to predict the disease status. Observe, however, that the prediction goes in the anticausal direction, since it is the structural brain differences shown by individuals at different stages of the disease that cause the observed patterns and intensity of pixels in the images, and not the other way around. In this sense, the observed patterns and intensity of pixels observed in neuroimages represent ``symptoms" of the neurologic disease in the same way as the tremor patterns captured by accelerometers represent symptoms of Parkinson's disease. Hence, the causal graph underlying diagnostic neuroimage applications (potentially influenced by selection biases) is given by the DAG,
\begin{figure}[!h]
$$
\xymatrix@-1.0pc{
& *+[F-:<10pt>]{A} \ar[dl] \ar[dr] \ar[r] & *+[F]{S} \\
*+[F-:<10pt>]{P_{ix}} && *+[F-:<10pt>]{Y} \ar[ll] \ar[u] \\
}
$$
\vskip -0.1in
  \caption{Causal diagram for neuroimage diagnostic applications.}
  \label{fig:examples.neuroimage}
\end{figure}
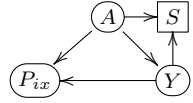
where $P_{ix}$ represents the images (or features extracted from the images), $Y$ represents the disease states, and $A$ represents a confounder such as age. (Here, $S$ represents a binary variable indicating the presence of a selection bias mechanism).

Note that other types of neuroimage applications, such as the mapping of different stimuli measured by fMRI, require more complicated modeling based on cyclic models (which we do not address in this paper).

\end{document}